\documentclass{article}

\usepackage{booktabs} 
\usepackage{amssymb}
\usepackage{amsmath}
\usepackage{wasysym}
\usepackage{epigraph}
\usepackage{graphics}
\usepackage{setspace}
\usepackage{fancybox} 
\newtheorem{definition}{Definition}
\newtheorem{theorem}{Theorem}
\newtheorem{lemma}{Lemma}
\newtheorem{claim}{Claim}
\newtheorem{corollary}{Corollary}

\newenvironment{proof}{\noindent{\sf Proof.}}{\hfill $\boxtimes\hspace{2mm}$\linebreak}
\newcommand{\qed}{\hfill $\boxtimes\hspace{1mm}$}
\renewcommand{\phi}{\varphi}
\renewcommand{\epsilon}{\varepsilon}

\newcommand{\K}{{\sf K}}
\renewcommand{\S}{{\sf S}}
\newcommand{\E}{{\sf H}}

\begin{document}
\title{Epistemic Strategies}
\title{Coalition Knowledge and Strategies}
\title{Knowing How to Achieve}
\title{Distributively Knowing How to Achieve}
\title{Knowing How to Achieve Together}
\title{Together We Know How to Achieve:\\ \vspace{1mm} \Large An Epistemic Logic of Know-How}

\author{Pavel Naumov \and Jia Tao}



\maketitle

\begin{abstract}
The existence of a coalition strategy to achieve a goal does not necessarily mean that the coalition has enough information to know how to follow the strategy. Neither does it mean that the coalition knows that such a strategy exists. The article studies an interplay between the distributed knowledge, coalition strategies, and coalition ``know-how" strategies. The main technical result is a sound and complete trimodal logical system that describes the properties of this interplay.
\end{abstract}

\maketitle

\section{Introduction}


An agent $a$ comes to a fork in a road. There is a sign that says that one of the two roads leads to prosperity, another to death. The agent must take the fork, but she does not know which road leads where. Does the agent have a strategy to get to prosperity? On one hand, since one of the roads leads to prosperity, such a strategy clearly exists. We denote this fact by modal formula $\S_a p$, where statement $p$ is a claim of future prosperity. Furthermore, agent $a$ knows that such a strategy exists. We write this as $\K_a\S_a p$. Yet, the agent does not know what the strategy is and, thus, does not know how to use the strategy. We denote this by $\neg\E_a p$, where {\em know-how} modality $\E_a$ expresses the fact that agent $a$ knows how to achieve the goal based on the information available to her.  In this article we study the interplay between modality $\K$, representing {\em knowledge}, modality $\S$, representing the existence of a {\em strategy}, and modality $\E$, representing the existence of a {\em know-how strategy}. Our main result is a complete trimodal axiomatic system capturing properties of this interplay.

\subsection{Epistemic Transition Systems}

In this article we use epistemic transition systems to capture knowledge and strategic behavior. Informally, epistemic transition system is a directed labeled graph supplemented by an indistinguishability relation on vertices. For instance, our motivational example above can be captured by epistemic transition system $T_1$ depicted in Figure~\ref{intro-1 figure}. 
\begin{figure}[ht]
\begin{center}
\vspace{-2mm}
\scalebox{.6}{\includegraphics{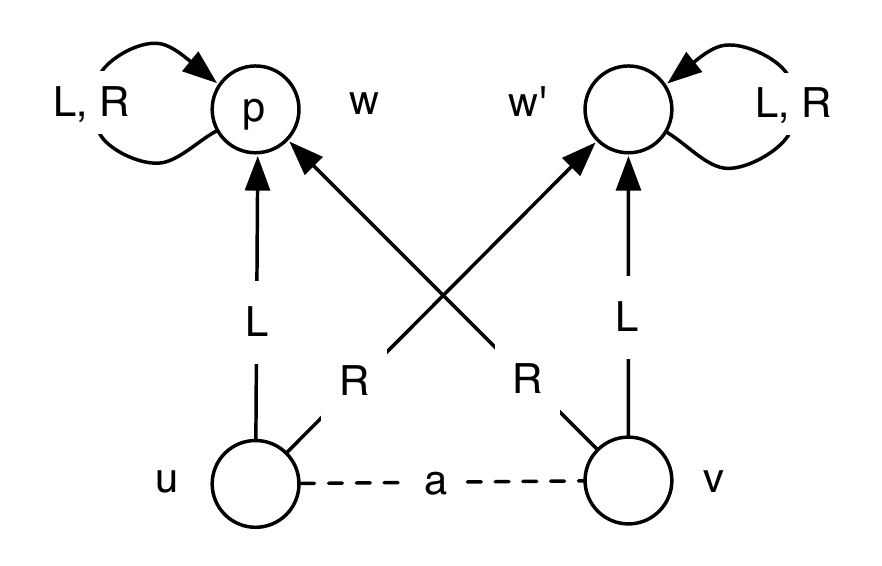}}
\vspace{0mm}
\caption{Epistemic transition system $T_1$.}\label{intro-1 figure}
\vspace{-2mm}
\end{center}
\vspace{-2mm}
\end{figure}
In this system state $w$ represents the prosperity and state $w'$ represents death. The original state is $u$, but it is indistinguishable by the agent $a$ from state $v$. Arrows on the diagram represent possible transitions between the states. Labels on the arrows represent the choices that the agents make during the transition. For example, if in state $u$ agent chooses left (L) road, she will transition to the prosperity state $w$ and if she chooses right (R) road, she will transition to the death state $w'$. In another epistemic state $v$, these roads lead the other way around. States $u$ and $v$ are not distinguishable by agent $a$, which is shown by the dashed line between these two states. In state $u$ as well as state $v$ the agent has a strategy to transition to the state of prosperity: $u\Vdash\S_a p$ and $v\Vdash\S_a p$. In the case of state $u$ this strategy is L, in the case of state $v$ the strategy is R. Since the agent cannot distinguish states $u$ and $v$, in both of these states she does not have a know-how strategy to reach prosperity: $u\nVdash\E_a p$ and $v\nVdash\E_a p$. At the same time, since formula $\S_a p$ is satisfied in all states  indistinguishable to agent $a$ from state $u$, we can claim that $u\Vdash\K_a\S_a p$ and, similarly, $v\Vdash\K_a\S_a p$. 

\begin{figure}[ht]
\begin{center}
\vspace{-2mm}
\scalebox{.6}{\includegraphics{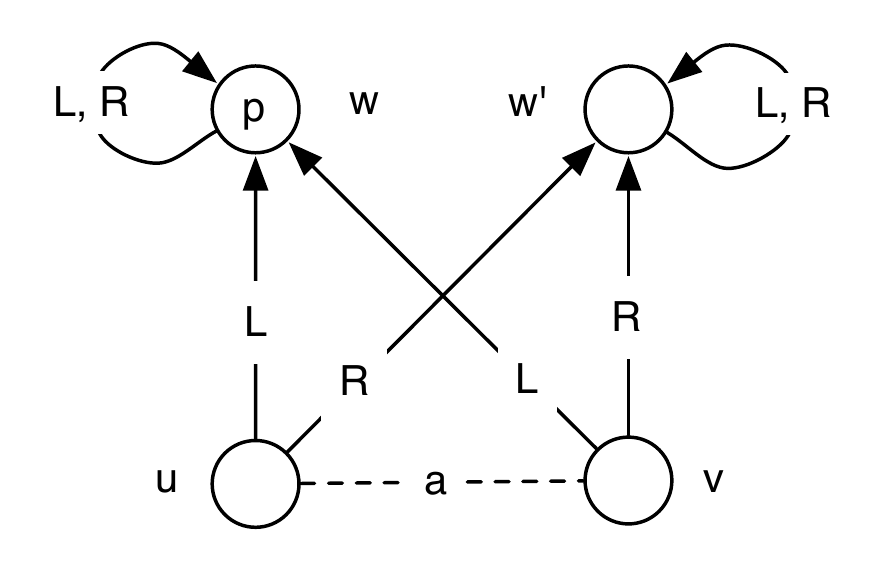}}
\vspace{0mm}
\caption{Epistemic transition system $T_2$.}\label{intro-2 figure}
\vspace{-2mm}
\end{center}
\vspace{-2mm}
\end{figure}

As our second example, let us consider the epistemic transition system $T_2$ obtained from $T_1$ by swapping labels on transitions from $v$ to $w$ and from $v$ to $w'$, see Figure~\ref{intro-2 figure}. Although in system $T_2$ agent $a$ still cannot distinguish states $u$ and $v$, she has a know-how strategy from either of these states to reach state $w$. We write this as $u\Vdash\E_a p$ and $v\Vdash\E_a p$. The strategy is to choose L. This strategy is know-how because it does not require to make different choices in the states that the agent cannot distinguish. 

\subsection{Imperfect Recall}
For the next example, we consider a transition system $T_3$ obtained from system $T_1$ by adding a new epistemic state $s$. From state $s$, agent $a$ can choose label L to reach state $u$ or choose label R to reach state $v$. Since proposition $q$ is satisfied in state $u$, agent $a$ has a know-how strategy to transition from state $s$ to a state (namely, state $u$) where $q$ is satisfied. Therefore, $s\Vdash\E_a q$. 

\begin{figure}[ht]
\begin{center}
\vspace{-2mm}
\scalebox{.6}{\includegraphics{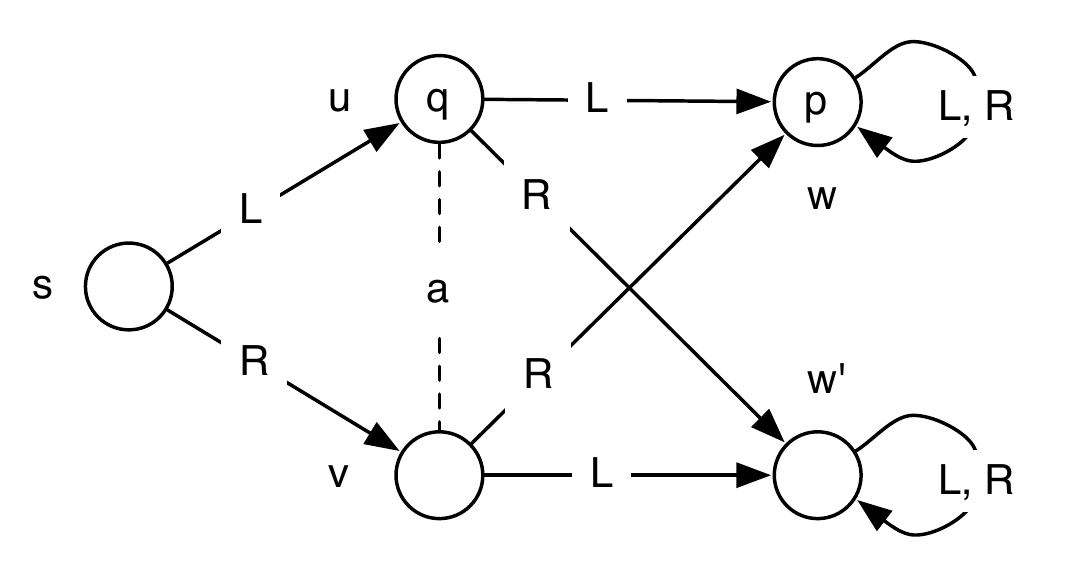}}
\vspace{0mm}
\caption{Epistemic transition system $T_3$.}\label{intro-3 figure}
\vspace{-2mm}
\end{center}
\vspace{-2mm}
\end{figure}

A more interesting question is whether $s\Vdash\E_a\E_a p$ is true. In other words, does agent $a$ know how to transition from state $s$ to a state in which she knows how to transition to another state in which $p$ is satisfied? One might think that such a strategy indeed exists: in state $s$ agent $a$ chooses  label L to transition to state $u$. Since there is no transition labeled by L that leads from state $s$ to state $v$, upon ending the first transition the agent would know that she is in state $u$, where she needs to choose label L to transition to state $w$. This argument, however, is based on the assumption that agent $a$ has a perfect recall. Namely, agent $a$ in state $u$ remembers the choice that she made in the previous state. We assume that the agents do not have a perfect recall and that an epistemic state description captures whatever memories the agent has in this state. In other words, in this article we assume that the only knowledge that an agent possesses is the knowledge captured by the indistinguishability relation on the epistemic states. Given this assumption, upon reaching the state $u$ (indistinguishable from state $v$) agent $a$ knows that there {\em exists} a choice that she can make to transition to state in which $p$ is satisfied: $s\Vdash\E_a\S_a p$. However, she does not know which choice (L or R) it is: $s\nVdash\E_a\E_a p$. 

\subsection{Multiagent Setting}

\begin{figure}[ht]
\begin{center}
\vspace{-2mm}
\scalebox{.6}{\includegraphics{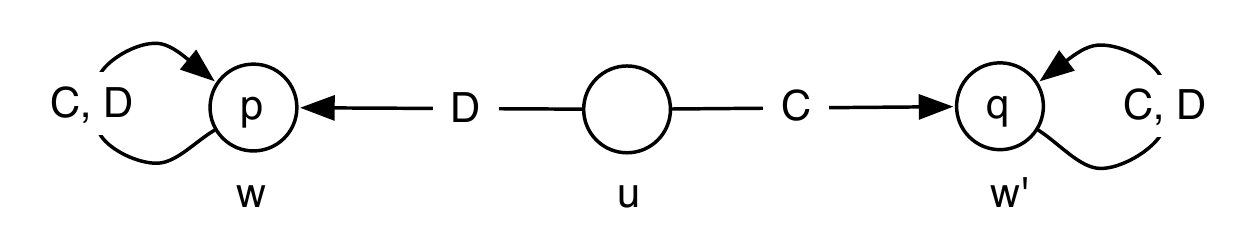}}
\vspace{0mm}
\caption{Epistemic transition system $T_4$.}\label{intro-4 figure}
\vspace{-2mm}
\end{center}
\vspace{-2mm}
\end{figure}

So far, we have assumed that only agent $a$ has an influence on which transition the system takes. In transition system $T_4$ depicted in Figure~\ref{intro-4 figure}, we introduce another agent $b$ and assume both agents $a$ and $b$ have influence on the transitions. In each state, the system takes the transition labeled D by default unless there is a consensus of agents $a$ and $b$ to take the transition labeled C. In such a setting, each agent has a strategy to transition system from state $u$ into state $w$ by voting D, but neither of them alone has a strategy to transition from state $u$ to state $w'$ because such a transition requires the consensus of both agents. Thus, $u\Vdash\S_a p\wedge \S_b p\wedge \neg\S_a q\wedge \neg\S_b q$. Additionally, both agents know how to transition the system from state $u$ into state $w$, they just need to vote D. Therefore, $u\Vdash\E_a p\wedge \E_b p$.

\begin{figure}[ht]
\begin{center}
\vspace{-2mm}
\scalebox{.6}{\includegraphics{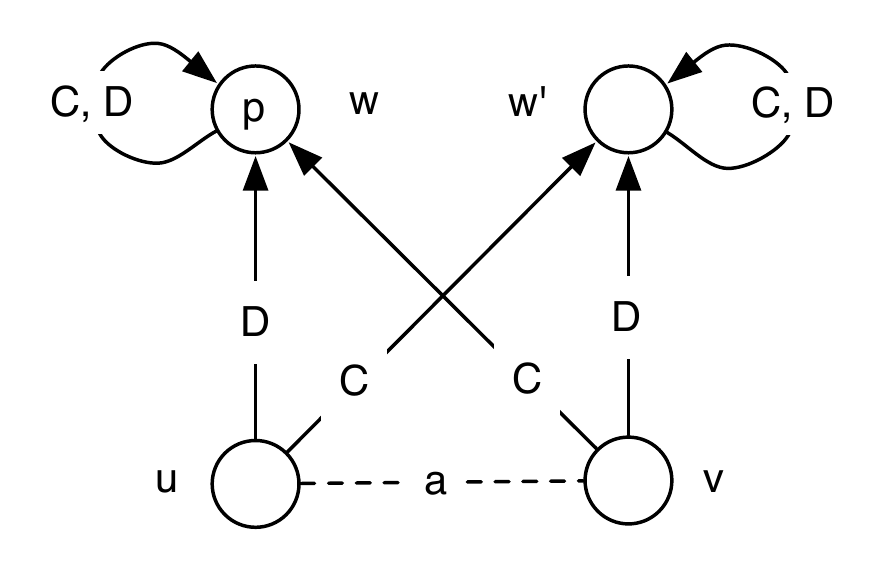}}
\vspace{0mm}
\caption{Epistemic transition system $T_5$.}\label{intro-5 figure}
\vspace{-2mm}
\end{center}
\vspace{-2mm}
\end{figure}

In Figure~\ref{intro-5 figure}, we show a more complicated transition system obtained from $T_1$ by renaming label L to D and renaming label R to C. Same as in transition system $T_4$, we assume that there are two agents $a$ and $b$ voting on the system transition. We also assume that agent $a$ cannot distinguish states $u$ and $v$ while agent $b$ can. By default, the system takes the transition labeled D unless there is a consensus to take transition labeled C. As a result, agent $a$ has a strategy (namely, vote D) in state $u$ to transition system to state $w$, but because agent $a$ cannot distinguish state $u$ from state $v$, not only does she not know how to do this, but she is not aware that such a strategy exists: $u\Vdash\S_a p\wedge\neg\E_a p \wedge\neg\K_a\S_a p$. Agent $b$, however, not only has a strategy to transition the system from state $u$ to state $w$, but also knows how to achieve this: $u\Vdash \E_b p$.

\subsection{Coalitions}

We have talked about strategies, know-hows, and knowledge of individual agents. In this article we consider knowledge, strategies, and know-how strategies of coalitions. There are several forms of group knowledge that have been studied before. The two most popular of them are common knowledge and distributed knowledge~\cite{fhmv95}. Different contexts call for different forms of group knowledge.

As illustrated in the famous Two Generals' Problem~\cite{aeh75sigop,g78os} where communication channels between the agents are unreliable, establishing a common knowledge between agents might be essential for having a strategy. 

In some settings, the distinction between common and distributed knowledge is insignificant. For example, if members of a political fraction get together to share {\em all} their information and to develop a common strategy, then the distributed knowledge of the members becomes the common knowledge of the fraction during the in-person meeting.

Finally, in some other situations the distributed knowledge makes more sense than the common knowledge. For example, if a panel of experts is formed to develop a strategy, then this panel achieves the best result if it relies on the combined knowledge of its members rather than on their common knowledge.

In this article we focus on distributed coalition knowledge and distributed-know-how strategies. We leave the common knowledge for the future research.

To illustrate how distributed knowledge of coalitions interacts with strategies and know-hows, consider epistemic transition system $T_6$ depicted in Figure~\ref{intro-6 figure}. In this system, agents $a$ and $b$ cannot distinguish states $u$ and $v$ while agents $b$ and $c$ cannot distinguish states $v$ and $u'$. In every state, each of agents $a$, $b$ and $c$ votes either L or R, and the system transitions according to the majority vote. In such a setting, any coalition of two agents can fully control the transitions of the system. 

\begin{figure}[ht]
\begin{center}
\vspace{-2mm}
\scalebox{.6}{\includegraphics{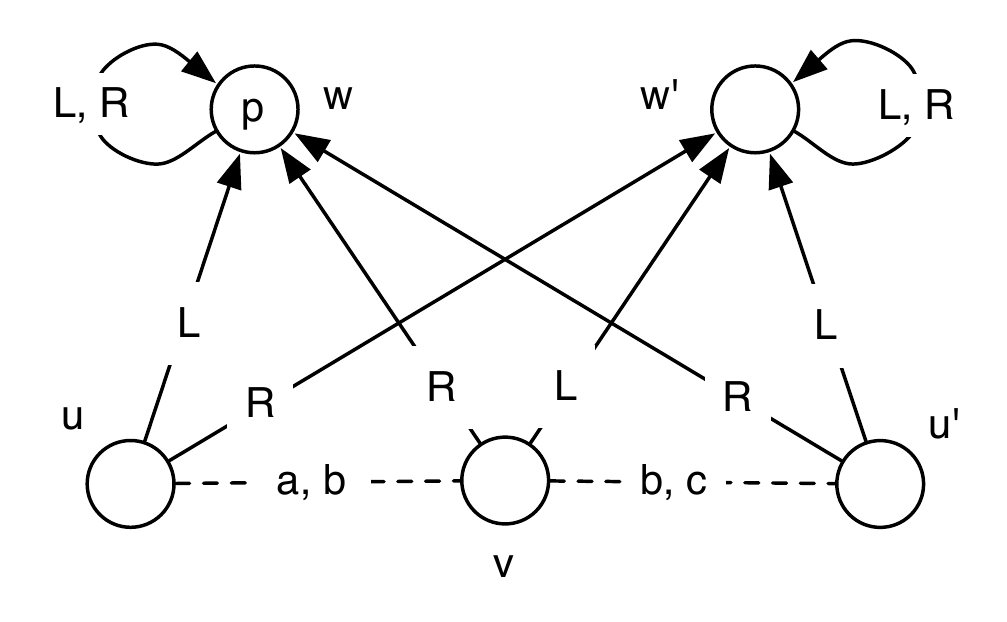}}
\vspace{0mm}
\caption{Epistemic transition system $T_6$.}\label{intro-6 figure}
\vspace{-2mm}
\end{center}
\vspace{-2mm}
\end{figure}

For example, by both voting L, agents $a$ and $b$ form a coalition $\{a,b\}$ that forces the system to transition from state $u$ to state $w$ no matter how agent $c$ votes. Since proposition $p$ is satisfied in state $w$, we write $u\Vdash\S_{\{a,b\}} p$, or simply $u\Vdash\S_{a,b} p$. Similarly, coalition $\{a,b\}$ can vote R to force the system to transition from state $v$ to state $w$. Therefore, coalition $\{a,b\}$ has strategies to achieve $p$ in states $u$ and $v$, but the strategies are different. Since they cannot distinguish states $u$ and $v$, agents $a$ and $b$ know that they have a strategy to achieve $p$, but they do \emph{not} know how to achieve $p$. In our notations, $v\Vdash S_{a,b}p\wedge \K_{a,b}S_{a,b}p \wedge \neg\E_{a,b} p$.   

On the other hand, although agents $b$ and $c$ cannot distinguish states $v$ and $u'$, by both voting R in either of states $v$ and $u'$, they form a coalition $\{b, c\}$ that forces the system to transition to state $w$ where $p$ is satisfied. Therefore, in any of states $v$ and $u'$, they not only have a strategy to achieve $p$, but also know that they have such a strategy, and more importantly, they know how to achieve $p$, that is, $v\Vdash\E_{b,c} p$.

\subsection{Nondeterministic Transitions}

In all the examples that we have discussed so far, given any state in a system, agents' votes uniquely determine the transition of the system. Our framework also allows nondeterministic transitions. Consider transition system $T_7$ depicted in Figure~\ref{intro-7 figure}. In this system, there are two agents $a$ and $b$ who can vote either C or D. If both agents vote C, then the system takes one of the consensus transitions labeled with C. Otherwise, the system takes the transition labeled with D. Note that there are two consensus transitions starting from state $u$. Therefore, even if both agents vote C, they do not have a strategy to achieve $p$, i.e., $u\nVdash\S_{a,b}p$. However, they can achieve $p\vee q$. Moreover, since all agents can distinguish all states, we have $u \Vdash\E_{a,b}(p\vee q)$.

\begin{figure}[ht]
\begin{center}
\vspace{-2mm}
\scalebox{.6}{\includegraphics{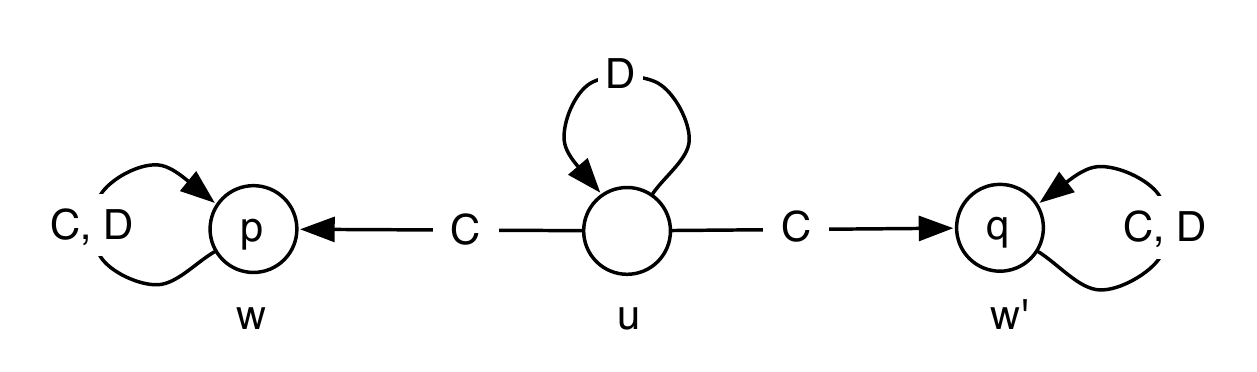}}
\vspace{0mm}
\caption{Epistemic transition system $T_7$.}\label{intro-7 figure}
\vspace{-2mm}
\end{center}
\vspace{-2mm}
\end{figure}

\subsection{Universal Principles}

In the examples above we focused on specific properties that were either satisfied or not satisfied in particular states of epistemic transition systems $T_1$ through $T_7$.
In this article, we study properties that are satisfied in all states of all epistemic transition systems. Our main result is a sound and complete axiomatization of all such properties. We finish the introduction with an informal discussion of these properties.

\paragraph{Properties of Single Modalities}
Knowledge modality $K_C$ satisfies the axioms of epistemic logic S5 with distributed knowledge. Both strategic modality $S_C$ and know-how modality $\E_C$ satisfy cooperation properties~\cite{p01illc,p02}:  
\begin{eqnarray}
&\S_C(\phi\to\psi)\to(\S_D\phi\to\S_{C\cup D}\psi), \mbox{ where } C\cap D=\varnothing,\label{s coop}\\
&\E_C(\phi\to\psi)\to(\E_D\phi\to\E_{C\cup D}\psi), \mbox{ where } C\cap D=\varnothing.\label{e coop}
\end{eqnarray}
They also satisfy monotonicity properties
\begin{eqnarray*}
\S_C\phi\to\S_D\phi, \mbox{ where } C\subseteq D,\\
\E_C\phi\to\E_D\phi, \mbox{ where } C\subseteq D.
\end{eqnarray*}
The two monotonicity properties are not among the axioms of our logical system because, as we show in Lemma~\ref{subset lemma S} and Lemma~\ref{subset lemma E}, they are derivable.

\paragraph{Properties of Interplay}

Note that $w\Vdash\E_C\phi$ means that coalition $C$ has the same strategy to achieve $\phi$ in all epistemic states indistinguishable by the coalition from state $w$. Hence, the following principle is universally true:
\begin{equation}\label{st pos intro}
    \E_C\phi\to K_C\E_C\phi.
\end{equation}
Similarly, $w\Vdash\neg\E_C\phi$ means that coalition $C$ does not have the same strategy to achieve $\phi$ in all epistemic states indistinguishable by the coalition from state $w$. Thus,
\begin{equation}\label{st neg intro}
    \neg\E_C\phi\to K_C\neg\E_C\phi.
\end{equation}
We call properties~(\ref{st pos intro}) and (\ref{st neg intro}) {\em strategic positive introspection} and {\em strategic negative introspection}, respectively. 
The strategic negative introspection is one of our axioms. Just as how the positive introspection principle follows from the rest of the axioms in S5 (see Lemma~\ref{positive introspection lemma}), the strategic positive introspection principle is also derivable (see Lemma~\ref{strategic positive introspection lemma}).

Whenever a coalition knows how to achieve something, there should exist a strategy for the coalition to achieve. In our notation,
\begin{equation}\label{st truth}
    \E_C\phi\to\S_C\phi.
\end{equation}
We call this formula {\em strategic truth} property and it is one of the axioms of our logical system.

The last two axioms of our logical system deal with empty coalitions. First of all, if formula $\K_\varnothing\phi$ is satisfied in an epistemic state of our transition system, then formula $\phi$ must be satisfied in every state of this system. Thus, even empty coalition has a trivial strategy to achieve $\phi$:
\begin{equation}\label{empty coal}
    \K_\varnothing\phi\to\E_\varnothing\phi. 
\end{equation}
We call this property {\em empty coalition} principle. In this article we assume that an epistemic transition system never halts. That is, in every state of the system no matter what the outcome of the vote is, there is always a next state for this vote. This restriction on the transition systems yields property 
\begin{equation}\label{nonerm}
    \neg\S_C\bot. 
\end{equation}
that we call {\em nontermination} principle.

Let us now turn to the most interesting and perhaps most unexpected property of interplay. Note that $\S_\varnothing\phi$ means that an empty coalition has a strategy to achieve $\phi$. Since the empty coalition has no members, nobody has to vote in a particular way. Statement $\phi$ is guaranteed to happen anyway. Thus, statement $\S_\varnothing\phi$ simply means that statement $\phi$ is unavoidably satisfied after any single transition. 

\begin{figure}[ht]
\begin{center}
\vspace{-2mm}
\scalebox{.6}{\includegraphics{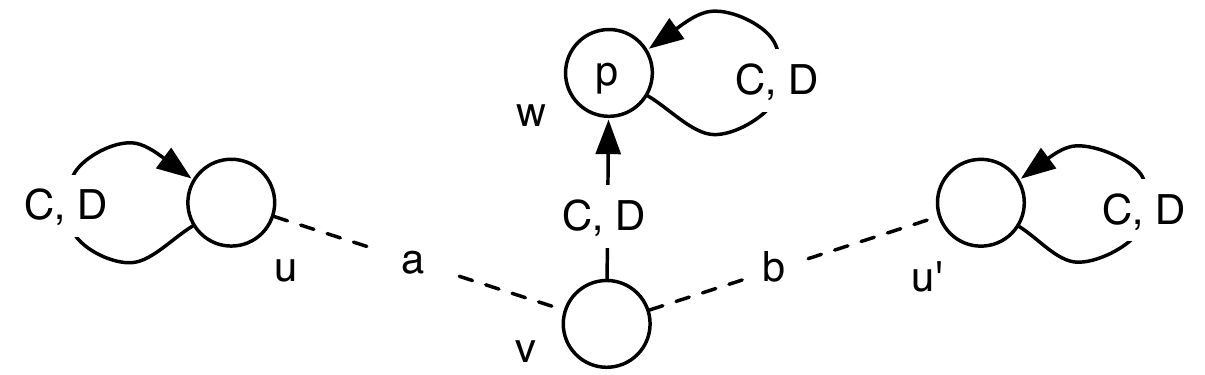}}
\vspace{0mm}
\caption{Epistemic transition system $T_8$.}\label{intro-8 figure}
\vspace{-2mm}
\end{center}
\vspace{-2mm}
\end{figure}
For example, consider an epistemic transition system depicted in Figure~\ref{intro-8 figure}. As in some of our earlier examples, this system has agents $a$ and $b$ who vote either C or D. If both agents vote C, then the system takes one of the consensus transitions labeled with C. Otherwise, the system takes the default transition labeled with D. Note that in state $v$ it is guaranteed that statement $p$ will happen after a single transition. Thus, $v\Vdash\S_\varnothing p$. At the same time, neither agent $a$ nor agent $b$ knows about this because they cannot distinguish state $v$ from states $u$ and $u'$ respectively. Thus, $v\Vdash\neg\K_a\S_\varnothing p \wedge \neg\K_b\S_\varnothing p$. 

In the same transition system $T_8$, agents $a$ and $b$ together can distinguish state $v$ from states $u$ and $u'$. Thus, $v\Vdash\K_{a,b}\S_\varnothing p$. 
In general,  statement $\K_C\S_\varnothing\phi$ means that not only $\phi$ is unavoidable, but coalition $C$ knows about it. Thus, coalition $C$ has a know-how strategy to achieve $\phi$:
$$
\K_C\S_\varnothing\phi\to \E_C\phi.
$$
In fact, the coalition would achieve the result no matter which strategy it uses. Coalition $C$ can even use a strategy that simultaneously achieves another result in addition to $\phi$: 
\begin{equation*}
    \K_C\S_\varnothing \phi\wedge \E_C\psi \to\E_C(\phi \wedge \psi).
\end{equation*}
In our logical system we use an equivalent form of the above principle that is stated using only implication:
\begin{equation}\label{ep determ}
    \E_C(\phi\to\psi)\to(\K_C\S_\varnothing \phi\to\E_C\psi).
\end{equation}
We call this property {\em epistemic determinicity} principle. Properties~(\ref{s coop}), (\ref{e coop}), (\ref{st neg intro}), (\ref{st truth}), (\ref{empty coal}), (\ref{nonerm}), and (\ref{ep determ}), together with axioms of epistemic logic S5 with distributed knowledge and propositional tautologies constitute the axioms of our sound and complete logical system. 

\subsection{Literature Review}

Logics of coalition power were developed by Marc Pauly~\cite{p01illc,p02}, who also proved the completeness of the basic logic of coalition power. 
Pauly's approach has been widely studied in the literature~\cite{g01tark,vw05ai,b07ijcai,sgvw06aamas,abvs10jal,avw09ai,b14sr}. An alternative logical system  was proposed by More and Naumov~\cite{mn12tocl}. 

Alur, Henzinger, and Kupferman introduced Alternating-Time Temporal Logic (ATL) that combines temporal and coalition modalities~\cite{ahk02}.
Van der Hoek and Wooldridge proposed to combine ATL with epistemic modality to form Alternating-Time Temporal Epistemic Logic~\cite{vw03sl}. They did not prove the completeness theorem for the proposed logical system. 

{\AA}gotnes and Alechina proposed a complete logical system that combines the coalition power and epistemic modalities~\cite{aa12aamas}. Since this system does not have epistemic requirements on strategies, it does not contain any axioms describing the interplay of these modalities. 

Know-how strategies were studied before under different names. While Jamroga and {\AA}gotnes talked about ``knowledge to identify and execute a strategy"~\cite{ja07jancl},  Jamroga and van der Hoek discussed ``difference between an agent knowing that he has a suitable strategy and knowing the strategy itself"~\cite{jv04fm}. Van Benthem called such strategies ``uniform"~\cite{v01ber}. Wang gave a complete axiomatization of ``knowing how" as a binary modality~\cite{w15lori,w17synthese}, but his logical system does not include the knowledge modality.  

In our AAMAS paper, we investigated coalition strategies to enforce a condition indefinitely~\cite{nt17aamas}. Such strategies are similar to ``goal maintenance" strategies in Pauly's ``extended coalition logic"~\cite[p. 80]{p01illc}. We focused on ``executable" and ``verifiable" strategies. Using the language of the current article, executability means that a coalition remains ``in the know-how" throughout the execution of the strategy. Verifiability means that the coalition can verify that the enforced condition remains true. In the notations of the current article, the existence of a verifiable strategy could be expressed as $\S_C\K_C\phi$.  In~\cite{nt17aamas}, we provided a complete logical system that describes the interplay between the modality representing the existence of an ``executable" and ``verifiable" coalition strategy to enforce and the modality representing knowledge. This system can prove principles similar to the strategic positive introspection~(\ref{st pos intro}) and the strategic negative introspection~(\ref{st neg intro}) mentioned above. A similar complete logical system in a {\em single-agent} setting for strategies to {achieve} a goal in multiple steps rather than to {maintain} a goal is developed by Fervari, Herzig,  Li,  and  Wang~\cite{fhlw17ijcai}.   

In the current article, we combine know-how modality $\E$ with strategic modality $\S$ and epistemic modality $\K$. The proof of the completeness theorem is significantly more challenging than in \cite{nt17aamas,fhlw17ijcai}. It employs new techniques that construct pairs of maximal consistent sets in ``harmony" and in ``complete harmony". See Section~\ref{harmony subsection} and Section~\ref{complete harmony subsection} for details. An extended abstract of this article, without proofs, appeared as~\cite{nt17tark}.

\subsection{Outline}

This article is organized as follows. In Section~\ref{syntax and semantics} we introduce formal syntax and semantics of our logical system. In Section~\ref{axioms section} we list axioms and inference rules of the system. Section~\ref{examples section} provides examples of formal proofs in our logical systems. 
Proofs of the soundness and the completeness are given in Section~\ref{soundness section} and Section~\ref{completeness section} respectively. Section~\ref{conclusion section} concludes the article.

The key part of the proof of the completeness is the construction of a pair of sets in complete harmony. We discuss the intuition behind this construction and introduce the notion of harmony in Section~\ref{harmony subsection}. The notion of complete harmony is introduced in Section~\ref{complete harmony subsection}.

\section{Syntax and Semantics}\label{syntax and semantics}

In this section we present the formal syntax and semantics of our logical system given a fixed finite set of agents $\mathcal{A}$. Epistemic transition system could be thought of as a Kripke model of modal logic S5 with distributed knowledge to which we add transitions controlled by a vote aggregation mechanism. Examples of vote aggregation mechanisms that we have considered in the introduction are the consensus/default mechanism and the majority vote mechanism. Unlike the introductory examples, in the general definition below we assume that at different states the mechanism might use different rules for vote aggregation. The only restriction on the mechanism that we introduce is that there should be at least one possible transition that the system can take no matter what the votes are. In other words, we assume that the system can never halt.  

For any set of votes $V$, by $V^\mathcal{A}$ we mean the set of all functions from set $\mathcal{A}$ to set $V$. Alternatively, the set $V^\mathcal{A}$ could be thought of as a set of tuples of elements of $V$ indexed by elements of $\mathcal{A}$.

\begin{definition}\label{transition system}
A tuple $(W,\{\sim_a\}_{a\in \mathcal{A}},V,M,\pi)$ is called an epistemic transition system, where
\begin{enumerate}
    \item $W$ is a set of epistemic states,
    \item $\sim_a$ is an indistinguishability equivalence relation on $W$ for each $a\in\mathcal{A}$,
    \item $V$ is a nonempty set called ``domain of choices", 
    \item $M\subseteq W\times V^\mathcal{A}\times W$ is an aggregation mechanism where for each $w\in W$ and each $\mathbf{s}\in V^\mathcal{A}$, there is $w'\in W$  such that $(w,\mathbf{s},w')\in M$,
    \item $\pi$ is a function that maps propositional variables into subsets of $W$.
\end{enumerate}
\end{definition}

\begin{definition}
A coalition is a subset of $\mathcal{A}$.
\end{definition}

Note that a coalition is always finite due to our assumption that the set of all agents $\mathcal{A}$ is finite. Informally, we say that two epistemic states are indistinguishable by a coalition $C$ if they are indistinguishable by every member of the coalition. Formally, coalition indistinguishability is defined as follows:

\begin{definition}\label{sim set}
For any epistemic states $w_1,w_2\in W$ and any coalition $C$, let $w_1\sim_C w_2$ if $w_1\sim_a w_2$ for each agent $a\in C$.
\end{definition}

\begin{corollary}\label{sim set corollary}
Relation $\sim_C$ is an equivalence relation on the set of states $W$ for each coalition $C$. 
\end{corollary}

By a strategy profile $\{s_a\}_{a\in C}$ of a coalition $C$ we mean a tuple that specifies vote $s_a\in V$ of each member $a\in C$. Since such a tuple can also be viewed as a function from set $C$ to set $V$, we denote the set of all strategy profiles of a coalition $C$ by $V^C$:

\begin{definition}\label{strategy}
Any tuple $\{s_a\}_{a\in C}\in V^C$ is called a strategy profile of coalition $C$.
\end{definition}

In addition to a fixed finite set of agents $\mathcal{A}$ we  also assume a fixed countable set of propositional variables. We use the assumption that this set is countable in the proof of Lemma~\ref{complete harmony lemma}. The language $\Phi$ of our formal logical system is specified in the next definition.   

\begin{definition}\label{Phi}
Let $\Phi$ be the minimal set of formulae such that
\begin{enumerate}
    \item $p\in\Phi$ for each propositional variable $p$,
    \item $\neg\phi,\phi\to\psi\in\Phi$ for all formulae $\phi,\psi\in\Phi$,
    \item $\K_C\phi,\S_C\phi,\E_C\phi\in\Phi$ for each coalition $C$ and each $\phi\in\Phi$.
\end{enumerate}
\end{definition}

In other words, language $\Phi$ is defined by the following grammar:
$$
\phi := p\;|\;\neg\phi\;|\;\phi\to\phi\;|\;\K_C\phi\;|\;\S_C\phi\;|\;\E_C\phi.
$$

By $\bot$ we denote the negation of a tautology. For example, we can assume that $\bot$ is $\neg(p\to p)$ for some fixed propositional variable $p$. 

According to Definition~\ref{transition system}, a mechanism specifies the transition that a system might take for any strategy profile of the set of {\em all} agents $\mathcal{A}$. It is sometimes convenient to consider transitions that are {\em consistent} with a given strategy profile $\mathbf s$ of a give coalition $C\subseteq \mathcal{A}$. We write $w\to_{\mathbf s}u$ if a transition from state $w$ to state $u$ is consistent with strategy profile $\mathbf s$. The formal definition is below.

\begin{definition}\label{single arrow}
For any epistemic states $w,u\in W$, any coalition $C$, and any strategy profile ${\mathbf s}=\{s_a\}_{a\in C}\in V^C$, we write $w\to_{\mathbf s}u$ if $(w,\mathbf{s'},u)\in M$ for some strategy profile $\mathbf{s'}=\{s'_a\}_{a\in\mathcal{A}}\in V^\mathcal{A}$  such that $s'_a=s_a$ for each $a\in C$.
\end{definition}
\begin{corollary}\label{empty coalition corollary}
For any strategy profile $\mathbf{s}$ of the empty coalition $\varnothing$, if there are a coalition $C$ and a strategy profile $\mathbf{s'}$ of coalition $C$ such that $w\to_\mathbf{s'} u$, then $w\to_{\mathbf s}u$.
\end{corollary}

 The next definition is the key definition of this article. It formally specifies the meaning of the three modalities in our logical system.



\begin{definition}\label{sat}
For any epistemic state $w\in W$ of a transition system $(W,\{\sim_a\}_{a\in \mathcal{A}},V,M,\pi)$ and any formula $\phi\in \Phi$, let relation $w\Vdash\phi$ be defined as follows
\begin{enumerate}
    \item $w\Vdash p$ if $w\in \pi(p)$ where $p$ is a propositional variable,
    \item $w\Vdash\neg\phi$ if $w\nVdash\phi$,
    \item $w\Vdash\phi\to\psi$ if $w\nVdash\phi$ or $w\Vdash\psi$,
    \item $w\Vdash \K_C\phi$ if $w'\Vdash\phi$ for each $w'\in W$ such that $w\sim_C w'$,
    \item $w\Vdash\S_C\phi$ if there is a strategy profile $\mathbf{s}\in V^C$ such that  $w\to_{\mathbf{s}} w'$ implies $w'\Vdash\phi$ for every $w'\in W$,
    \item $w\Vdash\E_C\phi$ if there is a strategy profile $\mathbf s\in V^C$ such that $w\sim_C w'$ and $w'\to_{\mathbf s} w''$ imply $w''\Vdash\phi$ for all $w',w''\in W$.
\end{enumerate}
\end{definition}

\section{Axioms}\label{axioms section}

In additional to propositional tautologies in language $\Phi$, our logical system consists of the following axioms. 
\begin{enumerate}
    \item Truth: $\K_C\phi\to\phi$,
    \item Negative Introspection: $\neg\K_C\phi\to\K_C\neg\K_C\phi$,
    \item Distributivity: $\K_C(\phi\to\psi)\to(\K_C\phi\to\K_{C}\psi)$,
    \item Monotonicity: $\K_C\phi\to \K_D\phi$, if $C\subseteq D$,
    \item Cooperation: $\S_C(\phi\to\psi)\to(\S_D\phi\to\S_{C\cup D}\psi)$, where $C\cap D=\varnothing$.
    \item Strategic Negative Introspection: $\neg\E_C\phi\to\K_C\neg\E_C\phi$,
    \item Epistemic Cooperation: $\E_C(\phi\to\psi)\to(\E_D\phi\to\E_{C\cup D}\psi)$,\\ where $C\cap D=\varnothing$,
    \item Strategic Truth: $\E_C\phi\to\S_C\phi$,
    \item Epistemic Determinicity: $\E_C(\phi\to\psi)\to(\K_C\S_\varnothing \phi\to\E_C\psi)$,
    \item Empty Coalition: $\K_\varnothing\phi\to\E_\varnothing \phi$,
    \item Nontermination: $\neg\S_C\bot$.
\end{enumerate}
We have discussed the informal meaning of these axioms in the introduction. In Section~\ref{soundness section} we formally prove the soundness of these axioms with respect to the semantics from Definition~\ref{sat}.

We write $\vdash \phi$ if formula $\phi$ is provable from the axioms of our logical system using 
Necessitation, Strategic Necessitation, and Modus Ponens inference rules:
$$
\dfrac{\phi}{\K_C\phi}
\hspace{10mm}
\dfrac{\phi}{\E_C\phi}
\hspace{10mm}
\dfrac{\phi,\hspace{5mm} \phi\to\psi}{\psi}.
$$
We write $X\vdash\phi$ if formula $\phi$ is provable from the theorems of our logical system and a set of additional axioms $X$ using only Modus Ponens inference rule.

\section{Derivation Examples}\label{examples section}

In this section we give examples of formal derivations in our logical system. In Lemma~\ref{strategic positive introspection lemma} we prove the strategic positive introspection principle~(\ref{st pos intro}) discussed in the introduction. The proof is similar to the proof of the epistemic positive introspection principle in Lemma~\ref{positive introspection lemma}.

\begin{lemma}\label{strategic positive introspection lemma}
$\vdash \E_C\phi\to\K_C\E_C\phi$.
\end{lemma}
\begin{proof}
Note that formula $\neg\E_C\phi\to\K_C\neg\E_C\phi$ is an instance of Strategic Negative Introspection axiom. Thus, $\vdash \neg\K_C\neg\E_C\phi\to \E_C\phi$ by the law of contrapositive in the propositional logic. Hence,
$\vdash \K_C(\neg\K_C\neg\E_C\phi\to \E_C\phi)$ by  Necessitation inference rule. Thus, by  Distributivity axiom and Modus Ponens inference rule, 
\begin{equation}\label{pos intro eq}
   \vdash  \K_C\neg\K_C\neg\E_C\phi\to \K_C\E_C\phi.
\end{equation}

At the same time, $\K_C\neg\E_C\phi\to\neg\E_C\phi$ is an instance of Truth axiom. Thus, $\vdash \E_C\phi\to\neg\K_C\neg\E_C\phi$ by contraposition. Hence, taking into account the following instance of Negative Introspection axiom $\neg\K_C\neg\E_C\phi\to\K_C\neg\K_C\neg\E_C\phi$,
one can conclude that $\vdash \E_C\phi\to\K_C\neg\K_C\neg\E_C\phi$. The latter, together with statement~(\ref{pos intro eq}), implies the statement of the lemma by the laws of propositional reasoning.
\end{proof}

In the next example, we show that the existence of a know-how strategy by a coalition implies that the coalition has a distributed knowledge of the existence of a strategy.

\begin{lemma}
$\vdash \E_C\phi\to\K_C\S_C\phi$.
\end{lemma}
\begin{proof}
By Strategic Truth axiom, $\vdash \E_C\phi\to\S_C\phi$. Hence, $\vdash \K_C(\E_C\phi\to\S_C\phi)$ by Necessitation inference rule. Thus, $\vdash \K_C\E_C\phi\to\K_C\S_C\phi$ by Distributivity axiom and Modus Ponens inference rule. At the same time, $\vdash \E_C\phi\to\K_C\E_C\phi$ by Lemma~\ref{strategic positive introspection lemma}. Therefore, $\vdash \E_C\phi\to\K_C\S_C\phi$ by the laws of propositional reasoning.
\end{proof}

The next lemma shows that the existence of a know-how strategy by a sub-coalition implies the existence of a know-how strategy by the entire coalition.

\begin{lemma}\label{subset lemma E}
$\vdash\E_C\phi\to \E_D\phi$, where $C\subseteq D$.
\end{lemma}
\begin{proof}
Note that $\phi\to\phi$ is a propositional tautology. Thus, $\vdash\phi\to\phi$. Hence, $\vdash\E_{D\setminus C}(\phi\to\phi)$ by Strategic Necessitation inference rule. At the same time, by Epistemic Cooperation axiom,
$
\vdash\E_{D\setminus C}(\phi\to\phi)\to(\E_C\phi\to\E_D\phi)
$
due to the assumption $C\subseteq D$.  Therefore, $\vdash\E_C\phi\to\E_D\phi$ by Modus Ponens inference rule.
\end{proof}

Although our logical system has three modalities, the system contains necessitation inference rules 
only for two of them. The lemma below shows that the necessitation rule for the third modality is admissible.
\begin{lemma}\label{s necessitation}
For each finite $C\subseteq \mathcal{A}$, inference rule $\dfrac{\phi}{\S_C\phi}$ is admissible in our logical system.
\end{lemma}
\begin{proof}
Assumption $\vdash\phi$ implies $\vdash\E_C\phi$ by Strategic Necessitation inference rule. Hence, $\vdash\S_C\phi$ by Strategic Truth axiom and Modus Ponens inference rule.
\end{proof}

The next result is a counterpart of Lemma~\ref{subset lemma E}. It states that the existence of a strategy by a sub-coalition implies the existence of a strategy by the entire coalition.

\begin{lemma}\label{subset lemma S}
$\vdash\S_C\phi\to \S_D\phi$, where $C\subseteq D$.
\end{lemma}
\begin{proof}
Note that $\phi\to\phi$ is a propositional tautology. Thus, $\vdash\phi\to\phi$. Hence,  $\vdash\S_{D\setminus C}(\phi\to\phi)$ by Lemma~\ref{s necessitation}.
At the same time, by Cooperation axiom,
$
\vdash\S_{D\setminus C}(\phi\to\phi)\to(\S_C\phi\to\S_D\phi)
$
due to the assumption $C\subseteq D$.  Therefore, $\vdash\S_C\phi\to\S_D\phi$ by Modus Ponens inference rule.
\end{proof}

\section{Soundness}\label{soundness section}

In this section we prove the soundness of our logical system. The proof of the soundness of multiagent S5 axioms and  inference rules is standard. Below we show the soundness of each of the remaining axioms and the Strategic Necessitation inference rule as a separate lemma. The soundness theorem for the whole logical system is stated at the end of this section as Theorem~\ref{soundness theorem}.

\begin{lemma}
If $w\Vdash\S_C(\phi\to\psi)$, $w\Vdash\S_D\phi$, and $C\cap D=\varnothing$, then $w\Vdash\S_{C\cup D}\psi$.
\end{lemma}
\begin{proof}
Suppose that $w\Vdash\S_C(\phi\to\psi)$. Then, by Definition~\ref{sat}, there is a strategy profile $\mathbf s^1=\{s^1_a\}_{a\in C}\in V^C$ such that $w'\Vdash\phi\to\psi$ for each $w'\in W$ where $w\to_{\mathbf s^1}w'$. Similarly, assumption $w\Vdash\S_D\phi$ implies that there is a strategy $\mathbf s^2=\{s^2_a\}_{a\in D}\in V^D$ such that $w'\Vdash\phi$ for each $w'\in W$ where $w\to_{\mathbf s^2}w'$. Let strategy profile ${\mathbf s}=\{s_a\}_{a\in C\cup D}$ be defined as follows:
$$
s_a=
\begin{cases}
s^1_a, & \mbox{ if } a\in C,\\
s^2_a, & \mbox{ if } a\in D.
\end{cases}
$$
Strategy profile $\mathbf s$ is well-defined due to the assumption $C\cap D=\varnothing$ of the lemma.

Consider any epistemic state $w'\in W$ such that $w\to_{\mathbf s}w'$. By Definition~\ref{sat}, it suffices to show that $w'\Vdash\psi$. Indeed, assumption $w\to_{\mathbf s}w'$, by Definition~\ref{single arrow}, implies that $w\to_{\mathbf s^1}w'$ and $w\to_{\mathbf s^2}w'$. Thus, $w'\Vdash\phi\to\psi$ and $w'\Vdash\phi$ by the choice of strategies $\mathbf s^1$ and $\mathbf s^2$. Therefore, $w'\Vdash\psi$ by Definition~\ref{sat}. 
\end{proof}

\begin{lemma}
If $w\Vdash\neg\E_C\phi$, then $w\Vdash\K_C\neg\E_C\phi$.
\end{lemma}
\begin{proof}
Consider any epistemic state $u\in W$ such that $w\sim_C u$. By Definition~\ref{sat}, it suffices to show that $u\nVdash\E_C\phi$. Assume the opposite. Thus, $u\Vdash\E_C\phi$. Then, again by Definition~\ref{sat}, there is a strategy profile $\mathbf{s}\in V^C$ where $u''\Vdash\phi$ for all $u',u''\in W$ such that $u\sim_C u'$ and $u'\to_\mathbf{s} u''$. Recall that  $w\sim_C u$. Thus, by Corollary~\ref{sim set corollary}, $u''\Vdash\phi$ for all $u',u''\in W$ such that $w\sim_C u'$ and $u'\to_\mathbf{s} u''$. Therefore, $w\Vdash\E_C\phi$, by Definition~\ref{sat}. The latter contradicts the assumption of the lemma.
\end{proof}

\begin{lemma}
If $w\Vdash\E_C(\phi\to\psi)$, $w\Vdash\E_D\phi$, and $C\cap D=\varnothing$, then $w\Vdash\E_{C\cup D}\psi$.
\end{lemma}
\begin{proof}
Suppose that $w\Vdash\E_C(\phi\to\psi)$. Thus, by Definition~\ref{sat}, there is a strategy profile $\mathbf s^1=\{s^1_a\}_{a\in C}\in V^C$ such that $w''\Vdash\phi\to\psi$ for all epistemic states $w',w''$ where $w\sim_C w'$ and $w'\to_{\mathbf s^1}w''$. Similarly, assumption $w\Vdash\E_D\phi$ implies that there is a strategy $\mathbf s^2=\{s^2_a\}_{a\in D}\in V^D$ such that $w''\Vdash\phi$ for all $w',w''$ where $w\sim_D w'$ and $w'\to_{\mathbf s^2}w''$. Let strategy profile ${\mathbf s}=\{s_a\}_{a\in C\cup D}$ be defined as follows:
$$
s_a=
\begin{cases}
s^1_a, & \mbox{ if } a\in C,\\
s^2_a, & \mbox{ if } a\in D.
\end{cases}
$$
Strategy profile $\mathbf s$ is well-defined due to the assumption $C\cap D=\varnothing$ of the lemma. 

Consider any epistemic states $w',w''\in W$ such that $w\sim_{C\cup D} w'$ and $w'\to_{\mathbf s}w''$. By Definition~\ref{sat}, it suffices to show that $w''\Vdash\psi$. Indeed, by Definition~\ref{sim set} assumption $w\sim_{C\cup D} w'$ implies that $w\sim_{C} w'$ and $w\sim_{D} w'$. At the same time, by Definition~\ref{single arrow}, assumption $w'\to_{\mathbf s}w''$ implies that $w'\to_{\mathbf s^1}w''$ and $w'\to_{\mathbf s^2}w''$. Thus, $w''\Vdash\phi\to\psi$ and $w''\Vdash\phi$ by the choice of strategies $\mathbf s^1$ and $\mathbf s^2$. Therefore, $w''\Vdash\psi$ by Definition~\ref{sat}. 
\end{proof}

\begin{lemma}
If $w\Vdash\E_C\phi$, then $w\Vdash\S_C\phi$.
\end{lemma}
\begin{proof}
Suppose that $w\Vdash\E_C\phi$. Thus, by Definition~\ref{sat}, there is a strategy profile $\mathbf{s}\in V^C$ such that $w''\Vdash\phi$ for all epistemic states $w',w''\in W$, where $w\sim_C w'$ and $w'\to_\mathbf{s} w''$. By Corollary~\ref{sim set corollary}, $w\sim_C w$. Hence, $w''\Vdash\phi$ for each epistemic state $w''\in W$, where $w\to_\mathbf{s} w''$. Therefore, $w\Vdash\S_C\phi$ by Definition~\ref{sat}.
\end{proof}

\begin{lemma}
If $w\Vdash\E_C(\phi\to\psi)$ and $w\Vdash\K_C\S_\varnothing\phi$, then $w\Vdash\E_C\psi$.
\end{lemma}
\begin{proof}
Suppose that $w\Vdash\E_C(\phi\to\psi)$. Thus, by Definition~\ref{sat}, there is a strategy profile $\mathbf{s}\in V^C$ such that $w''\Vdash\phi\to\psi$ for all epistemic states $w',w''\in W$ where $w\sim_C w'$ and $w'\to_\mathbf{s} w''$.

Consider any epistemic states $w'_0,w''_0\in W$ such that $w\sim_C w'_0$ and $w'_0\to_\mathbf{s} w''_0$. By Definition~\ref{sat}, it suffices to show that $w''_0\Vdash\psi$. 

Indeed, by Definition~\ref{sat}, the assumption $w\Vdash\K_C\S_\varnothing\phi$ together with $w\sim_C w'_0$ imply that $w'_0\Vdash\S_\varnothing\phi$. Hence, by Definition~\ref{sat}, there is a strategy profile $\mathbf{s'}$ of empty coalition $\varnothing$ such that $w''\Vdash\phi$ for each $w''$ where $w'_0\to_\mathbf{s'}w''$. Thus, $w''_0\Vdash\phi$ due to Corollary~\ref{empty coalition corollary} and $w'_0\to_\mathbf{s} w''_0$.
By the choice of strategy profile $\mathbf{s}$, statements
$w\sim_C w'_0$ and $w'_0\to_\mathbf{s} w''_0$ imply $w''_0\Vdash\phi\to\psi$. 
Finally, by Definition~\ref{sat}, statements $w''_0\Vdash\phi\to\psi$ and $w''_0\Vdash\phi$ imply that $w''_0\Vdash\psi$.
\end{proof}

\begin{lemma}
If $w\Vdash\K_\varnothing\phi$, then $w\Vdash\E_\varnothing\phi$.
\end{lemma}
\begin{proof}
Let $\mathbf{s}=\{s_a\}_{a\in\varnothing}$ be the empty strategy profile. Consider any epistemic states $w',w''\in W$ such that $w\sim_\varnothing w'$ and $w'\to_\mathbf{s} w''$. By Definition~\ref{sat}, it suffices to show that $w''\Vdash\phi$. Indeed $w\sim_\varnothing w''$ by Definition~\ref{sim set}. Therefore, $w''\Vdash\phi$ by assumption $w\Vdash\K_\varnothing\phi$ and Definition~\ref{sat}.
\end{proof}

\begin{lemma}
$w\nVdash\S_C\bot$. 
\end{lemma}
\begin{proof}
Suppose that $w\Vdash\S_C\bot$. Thus, by Definition~\ref{sat}, there is a strategy profile $\mathbf{s}=\{s_a\}_{a\in \mathcal{A}}\in V^C$ such that $u\Vdash\bot$ for each $u\in W$ where $w\to_{\mathbf{s}}u$. 

Note that by Definition~\ref{transition system}, the domain of choices $V$ is not empty. Thus, strategy profile $\mathbf{s}$ can be extended to a strategy profile $\mathbf{s}'=\{s'_a\}_{a\in \mathcal{A}}\in V^\mathcal{A}$ such that $s'_a=s_a$ for each $a\in C$. 

By Definition~\ref{transition system}, there must exist a state $w'\in W$ such that $(w,\mathbf{s}',w')\in M$. Hence, $w\to_{\mathbf{s}}w'$ by Definition~\ref{single arrow}. Therefore, $w'\Vdash\bot$ by the choice of strategy $\mathbf{s}$, which contradicts Definition~\ref{sat}.
\end{proof}

\begin{lemma}
If $w\Vdash\phi$ for any epistemic state $w\in W$ of an epistemic transition system $(W,\{\sim_a\}_{a\in \mathcal{A}},V,M,\pi)$, then $w\Vdash\S_C\phi$ for every epistemic state $w\in W$.
\end{lemma}
\begin{proof}
By Definition~\ref{transition system}, set $V$ is not empty. Let $v\in V$. Consider strategy profile $\mathbf{s}=\{s_a\}_{a\in C}$ of coalition $C$ such that $s_a=v$ for each $s\in C$. Note that $w'\Vdash\phi$ for each $w'\in W$ due to the assumption of the lemma. Therefore, $w\Vdash\S_C\phi$ by Definition~\ref{sat}. 
\end{proof}

Taken together, the lemmas above imply the soundness theorem for our logical system stated below.
\begin{theorem}\label{soundness theorem}
If $\vdash\phi$, then $w\Vdash\phi$ for each epistemic state $w\in W$ of each epistemic transition system $(W,\{\sim_a\}_{a\in \mathcal{A}},V,M,\pi)$. \qed
\end{theorem}

\section{Completeness}\label{completeness section}

This section is dedicated to the proof of the following completeness theorem for our logical system.

\begin{theorem}\label{completeness theorem}
If $w\Vdash\phi$ for each epistemic state $w$ of each epistemic transition system, then $\vdash\phi$.
\end{theorem}

\subsection{Positive Introspection}\label{positive introspection subsection}

The proof of Theorem~\ref{completeness theorem} is divided into several parts. In this section we prove the positive introspection principle for distributed knowledge modality from the rest of modality $\K$ axioms  in our logical system.  This is a well-known result that we reproduce to keep the presentation self-sufficient. The positive introspection principle is used later in the proof of the completeness.

\begin{lemma}\label{positive introspection lemma}
$\vdash \K_C\phi\to\K_C\K_C\phi$.
\end{lemma}
\begin{proof}
Formula $\neg\K_C\phi\to\K_C\neg\K_C\phi$ is an instance of Negative Introspection axiom. Thus, $\vdash \neg\K_C\neg\K_C\phi\to \K_C\phi$ by the law of contrapositive in the propositional logic. Hence,
$\vdash \K_C(\neg\K_C\neg\K_C\phi\to \K_C\phi)$ by Necessitation inference rule. Thus, by  Distributivity axiom and Modus Ponens inference rule, 
\begin{equation}\label{pos intro eq 2}
   \vdash \K_C\neg\K_C\neg\K_C\phi\to \K_C\K_C\phi.
\end{equation}

At the same time, $\K_C\neg\K_C\phi\to\neg\K_C\phi$ is an instance of Truth axiom. Thus, $\vdash \K_C\phi\to\neg\K_C\neg\K_C\phi$ by contraposition. Hence, taking into account the following instance of  Negative Introspection axiom $\neg\K_C\neg\K_C\phi\to\K_C\neg\K_C\neg\K_C\phi$,
one can conclude that $\vdash \K_C\phi\to\K_C\neg\K_C\neg\K_C\phi$. The latter, together with statement~(\ref{pos intro eq 2}), implies the statement of the lemma by the laws of propositional reasoning.
\end{proof}

\subsection{Consistent Sets of Formulae}\label{consistent sets section}

The proof of the completeness consists in constructing a canonical model in which states are maximal consistent sets of formulae. This is a standard technique in modal logic that we modified significantly to work in the setting of our logical system. The standard way to apply this technique to a modal operator $\Box$ is to create a ``child" state $w'$ such that $\neg\psi\in w'$ for each ``parent" state $w$ where $\neg\Box\psi\in w$. In the simplest case when $\Box$ is a distributed knowledge modality $\K_C$, the standard technique requires no modification and the construction of a ``child" state is based on the following lemma:

\begin{lemma}\label{k consist}
For any consistent set of formulae $X$, any formula $\neg\K_C\psi\in X$, and any formulae $\K_C\phi_1,\dots,\K_C\phi_n\in X$, the set of
formulae $\{\neg\psi,\phi_1,\dots,\phi_n\}$ is consistent.
\end{lemma}
\begin{proof}
Assume the opposite. Then, 
$
\phi_1,\dots,\phi_n\vdash\psi
$.
Thus, by the deduction theorem for propositional logic applied $n$ times,
$$\vdash\phi_1\to(\phi_2\to\dots(\phi_n\to\psi)\dots).$$    
Hence, by Necessitation inference rule,
$$\vdash\K_C(\phi_1\to(\phi_2\to\dots(\phi_n\to\psi)\dots)).$$   
By Distributivity axiom and Modus Ponens inference rule,
$$\K_C\phi_1\vdash\K_C(\phi_2\to\dots(\phi_n\to\psi)\dots).$$ 
By repeating the last step $(n-1)$ times,
$$\K_C\phi_1,\dots,\K_C\phi_n\vdash\K_C\psi.$$
Hence, $X\vdash \K_C\psi$ by the choice of formula $\K_C\phi_1,\dots,\K_C\phi_n$, which contradicts the consistency of the set $X$ due to the assumption $\neg\K_C\psi\in X$.
\end{proof}

If $\Box$ is the modality $\S_C$, then the standard technique needs to be modified. Namely, while $\neg\S_C\psi\in w$ means that coalition $C$ can not achieve goal $\psi$, its pairwise disjoint sub-coalitions $D_1,\dots, D_n\subseteq C$ might still achieve their own goals $\phi_1,\dots,\phi_n$.   An equivalent of Lemma~\ref{k consist} for modality $\S_C$ is the following statement.

\begin{lemma}\label{i neq j S}
For any consistent set of formulae $X$, and any subsets $D_1,\dots,D_n$  of a coalition $C$, any formula $\neg\S_C\psi\in X$, and any $\S_{D_1}\phi_1,\dots,\S_{D_n}\phi_n\in X$, if $D_i\cap D_j=\varnothing$ for all integers $i,j\le n$ such that $i\neq j$, then the set of formulae $\{\neg\psi,\phi_1,\dots,\phi_n\}$ is consistent.
\end{lemma}
\begin{proof}
Suppose that
$
\phi_1,\phi_2,\dots,\phi_n\vdash\psi.
$
Hence, by the deduction theorem for propositional logic applied $n$ times,
$$ 
\vdash\phi_1\to(\phi_2\to(\dots(\phi_n\to\psi)\dots)).
$$
Then, 
$
\vdash\S_{D_1}(\phi_1\to(\phi_2\to(\dots(\phi_n\to\psi)\dots)))
$
by Lemma~\ref{s necessitation}.
Hence, by Cooperation axiom and Modus Ponens inference rule,
$$
\vdash\S_{D_1}\phi_1\to\S_{\varnothing\cup D_1}(\phi_2\to(\dots(\phi_n\to\psi)\dots)).
$$
In other words,
$$
\vdash\S_{D_1}\phi_1\to\S_{D_1}(\phi_2\to(\dots(\phi_n\to\psi)\dots)).
$$
Then, by Modus Ponens inference rule,
$$
\S_{D_1}\phi_1\vdash\S_{D_1}(\phi_2\to(\dots(\phi_n\to\psi)\dots)).
$$
By Cooperation axiom and Modus Ponens inference rule, 
$$
\S_{D_1}\phi_1\vdash\S_{D_2}\phi_2\to\S_{D_1\cup D_2}(\dots(\phi_n\to\psi)\dots).
$$
Again, by Modus Ponens inference rule, 
$$
\S_{D_1}\phi_1,\S_{D_2}\phi_2 \vdash \S_{D_1\cup D_2}(\dots(\phi_n\to\psi)\dots).
$$
By repeating the previous steps $n-2$ times,
$$
\S_{D_1}\phi_1,\S_{D_2}\phi_2,\dots, \S_{D_n}\phi_n\vdash \S_{D_1\cup D_2\cup\dots\cup D_n}\psi.
$$

Recall that $\S_{D_1}\phi_1,\S_{D_2}\phi_2,\dots, \S_{D_n}\phi_n\in X$ by the assumption of the lemma. Thus,
$X\vdash \S_{D_1\cup D_2\cup\dots\cup D_n}\psi$. Therefore, $X\vdash \S_{C}\psi$ by Lemma~\ref{subset lemma S}. Since the set $X$ is consistent, the latter contradicts the assumption $\neg\S_C\psi\in X$ of the lemma.
\end{proof}

\subsection{Harmony}\label{harmony subsection}

If $\Box$ is the modality $\E_C$, then the standard technique needs even more significant modification. Namely, as it follows from Definition~\ref{sat}, assumption $\neg\E_C\psi\in w$ requires us to create not a single child of parent $w$, but two different children referred in Definition~\ref{sat} as states $w'$ and $w''$, see Figure~\ref{harmony figure}. Child $w'$ is a state of the system indistinguishable from state $w$ by coalition $C$. Child $w''$ is a state such that $\neg\psi\in w''$ and coalition $C$ cannot prevent the system to transition from $w'$ to $w''$.

\begin{figure}[ht]
\begin{center}
\vspace{-2mm}
\scalebox{.6}{\includegraphics{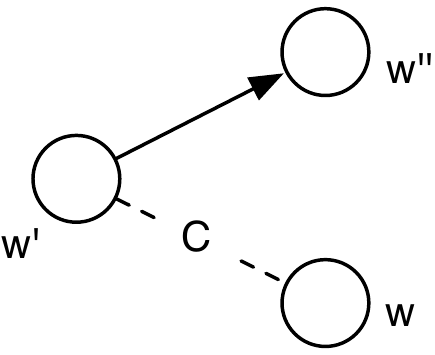}}
\vspace{0mm}
\caption{States $w'$ and $w'$ are maximal consistent sets of formulae in complete harmony.}\label{harmony figure}
\vspace{-2mm}
\end{center}
\vspace{-2mm}
\end{figure}

One might think that states $w'$ and $w''$ could be constructed in order: first state $w'$ and then state $w''$. It appears, however, that such an approach does not work because it does not guarantee that $\neg\psi\in w''$. To solve the issue, we construct states $w'$ and $w''$  simultaneously. While constructing states $w'$ and $w''$ as maximal consistent sets of formulae, it is important to maintain two relations between sets $w'$ and $w''$ that we call ``to be in harmony" and ``to be in complete harmony". In this section we define harmony relation and prove its basic properties. The next section is dedicated to the complete harmony relation.

Even though according to Definition~\ref{Phi} the language of our logical system only includes propositional connectives $\neg$ and $\to$, other connectives, including conjunction $\wedge$, can be defined in the standard way. By $\wedge Y$ we mean the conjunction of a finite set of formulae $Y$. If set $Y$ is a singleton, then $\wedge Y$ represents the single element of set $Y$. If set $Y$ is empty, then $\wedge Y$ is defined to be any propositional tautology.

\begin{definition}\label{harmony}
Pair $(X,Y)$ of sets of formulae is in harmony if
$X\nvdash \S_\varnothing\neg\wedge Y'$
for each finite set $Y'\subseteq Y$.
\end{definition}

\begin{lemma}\label{harmony con 1}
If pair $(X,Y)$ is in harmony, then set $X$ is consistent.
\end{lemma}
\begin{proof}
If set $X$ is not consistent, then any formula can be derived from it. In particular, $X\vdash \S_\varnothing\neg\wedge \varnothing$. Therefore, pair $(X,Y)$ is not in harmony by Definition~\ref{harmony}.
\end{proof}

\begin{lemma}\label{harmony con 2}
If pair $(X,Y)$ is in harmony, then set $Y$ is consistent.
\end{lemma}
\begin{proof}
Suppose that $Y$ is inconsistent. Then, there is a finite set $Y'\subseteq Y$ such that $\vdash\neg\wedge Y'$. Hence, $\vdash\S_\varnothing\neg\wedge Y'$ by Lemma~\ref{s necessitation}. Thus, $X\vdash\S_\varnothing\neg\wedge Y'$. Therefore, by Definition~\ref{harmony}, pair $(X,Y)$ is not in harmony.
\end{proof}

\begin{lemma}\label{harmony step}
For any $\phi\in\Phi$, if pair $(X,Y)$ is in harmony, then either pair $(X\cup\{\neg\S_\varnothing\phi\},Y)$ or pair $(X,Y\cup\{\phi\})$ is in harmony.
\end{lemma}


\begin{proof}
Suppose that neither pair $(X\cup\{\neg\S_\varnothing\phi\},Y)$ nor pair $(X,Y\cup\{\phi\})$ is in harmony. Then, by Definition~\ref{harmony}, there are finite sets $Y_1\subseteq Y$ and $Y_2\subseteq Y\cup\{\phi\}$ such that
\begin{equation}\label{january eq}
X,\neg\S_\varnothing\phi\vdash \S_\varnothing\neg\wedge Y_1
\end{equation}
and 
\begin{equation}\label{february eq}
X\vdash \S_\varnothing\neg\wedge Y_2.
\end{equation}

Formula $\neg\wedge Y_1\to \neg((\wedge Y_1)\wedge(\wedge (Y_2\setminus\{\phi\})))$ is a propositional tautology. Thus, $\vdash\S_\varnothing(\neg\wedge Y_1\to \neg((\wedge Y_1)\wedge(\wedge (Y_2\setminus\{\phi\}))))$ by Lemma~\ref{s necessitation}. Then, by Cooperation axiom, statement~(\ref{january eq}), and Modus Ponens inference rule,
$
X,\neg\S_\varnothing\phi\vdash \S_{\varnothing\cup\varnothing}\neg((\wedge Y_1)\wedge(\wedge (Y_2\setminus\{\phi\})))
$. In other words, 
\begin{equation}\label{march eq}
X,\neg\S_\varnothing\phi\vdash \S_\varnothing\neg((\wedge Y_1)\wedge(\wedge (Y_2\setminus\{\phi\}))).
\end{equation}

Finally, formula $\neg\wedge Y_2\to (\phi\to \neg((\wedge Y_1)\wedge(\wedge (Y_2\setminus\{\phi\}))))$ is also a propositional tautology. Thus, by Lemma~\ref{s necessitation}, 
$$\vdash\S_\varnothing(\neg\wedge Y_2\to (\phi\to \neg((\wedge Y_1)\wedge(\wedge (Y_2\setminus\{\phi\}))))).$$ 
Then, by Cooperation axiom, statement~(\ref{february eq}), and Modus Ponens inference rule,
$X\vdash \S_\varnothing(\phi\to \neg((\wedge Y_1)\wedge(\wedge (Y_2\setminus\{\phi\}))))$. Thus, by Cooperation axiom and Modus Ponens inference rule, 
$$X\vdash \S_\varnothing\phi\to \S_\varnothing\neg((\wedge Y_1)\wedge(\wedge (Y_2\setminus\{\phi\}))).$$
By Modus Ponens inference rule,
$$
X, \S_\varnothing\phi\vdash  \S_\varnothing\neg((\wedge Y_1)\wedge(\wedge (Y_2\setminus\{\phi\}))).
$$
Hence, $X\vdash \S_\varnothing\neg((\wedge Y_1)\wedge(\wedge (Y_2\setminus\{\phi\})))$ by statement~(\ref{march eq}) and the laws of propositional reasoning. Recall that $Y_1$ and $Y_2\setminus\{\phi\}$ are subsets of $Y$. Therefore, pair $(X,Y)$ is not in harmony by Definition~\ref{harmony}. 
\end{proof}

The next lemma is an equivalent of Lemma~\ref{k consist} and Lemma~\ref{i neq j S} for modality $\E_C$.

\begin{lemma}\label{i neq j E}
For any consistent set of formulae $X$, any formula $\neg\E_C\psi\in X$, and any function $f:C\to \Phi$, pair $(Y,Z)$ is in harmony, where
\begin{eqnarray*}
&&\hspace{-1cm} Y=\{\phi\;|\;\K_C\phi\in X\}, \mbox{ and}\\
&&\hspace{-1cm} Z=\{\neg\psi\}\cup\{\chi\;|\; \exists D\subseteq C\;(\E_D\chi\in X \wedge \forall a\in D\;(f(a)=\chi))\}.
\end{eqnarray*}
\end{lemma}
\begin{proof}
Suppose that pair $(Y,Z)$ is not in harmony. Thus, by Definition~\ref{harmony}, there is a finite $Z'\subseteq Z$ such that 
$Y\vdash\S_\varnothing\neg\wedge Z'$. Since a derivation uses only finitely many assumptions, there are formulae $K_C\phi_1,\K_C\phi_2\dots,\K_C\phi_n\in X$ such that
$$
\phi_1,\phi_2\dots,\phi_n\vdash \S_\varnothing\neg\wedge Z'.
$$
Then, by the deduction theorem for propositional logic applied $n$ times,
$$
\vdash\phi_1\to(\phi_2\to(\dots\to(\phi_n\to \S_\varnothing\neg\wedge Z')\dots)).
$$
Hence, by Necessitation inference rule,
$$
\vdash\K_C(\phi_1\to(\phi_2\to(\dots\to(\phi_n\to \S_\varnothing\neg\wedge Z')\dots))).
$$
Then, by Distributivity axiom and Modus Ponens inference rule,
$$
\vdash\K_C\phi_1\to\K_C(\phi_2\to(\dots\to(\phi_n\to \S_\varnothing\neg\wedge Z')\dots)).
$$
Thus, by Modus Ponens inference rule,
$$
\K_C\phi_1\vdash\K_C(\phi_2\to(\dots\to(\phi_n\to \S_\varnothing\neg\wedge Z')\dots)).
$$
By repeating the previous two steps $(n-1)$ times,
$$
\K_C\phi_1,\K_C\phi_2\dots, \K_C\phi_n\vdash \K_C\S_\varnothing\neg\wedge Z'.
$$
Hence, by the choice of formulae $K_C\phi_1,\K_C\phi_2,\dots,\K_C\phi_n$,
\begin{equation}\label{may eq}
X\vdash \K_C\S_\varnothing\neg\wedge Z'.
\end{equation}
Since set $Z'$ is a subset of set $Z$, by the choice of set $Z$, there must exist formulae $\E_{D_1}\chi_1,\dots,\E_{D_n}\chi_n\in X$ such that $D_1,\dots,D_n\subseteq C$,
\begin{equation}\label{chi eq}
    \forall i\le n\;\forall a\in D_i\;(f(a)=\chi_i),
\end{equation}
and
the following formula is a tautology, even if $\neg\psi\notin Z'$:
\begin{equation}\label{choice of chi}
\chi_1\to(\chi_2\to\dots(\chi_n\to(\neg\psi\to\wedge Z'))\dots).
\end{equation}
Without loss of generality, we can assume that formulae $\chi_1,\dots,\chi_n$ are {\em pairwise distinct}. 

\begin{claim}\label{disjoint claim}
$D_i\cap D_j=\varnothing$ for each $i,j\le n$ such that $i\neq j$.
\end{claim}
\noindent{\sc Proof of Claim.}
Suppose the opposite. Then, there is $a\in D_i\cap D_j$. Thus, $\chi_i=f(a)=\chi_j$ by statement~(\ref{chi eq}). This contradicts the assumption that formulae $\chi_1,\dots,\chi_n$ are pairwise distinct. 
\qed

Since formula~(\ref{choice of chi}) is a propositional tautology, by the law of contrapositive, the following formula is also a propositional tautology:
$$
\chi_1\to(\chi_2\to\dots(\chi_n\to(\neg\wedge Z'\to\psi))\dots).
$$
Thus, by Strategic Necessitation inference rule,
$$
\vdash\E_\varnothing(\chi_1\to(\chi_2\to\dots(\chi_n\to(\neg\wedge Z'\to\psi))\dots)).
$$
Hence, by Epistemic Cooperation axiom and Modus Ponens inference rule,
$$
\vdash\E_{D_1}\chi_{1}\to\E_{\varnothing\cup D_{1}}(\chi_{2}\to\dots(\chi_{n}\to(\neg\wedge Z'\to\psi))\dots).
$$
Then, by Modus Ponens inference rule,
$$
\E_{D_1}\chi_{1}\vdash\E_{ D_{1}}(\chi_{2}\to\dots(\chi_{n}\to(\neg\wedge Z'\to\psi))\dots).
$$
By Epistemic Cooperation axiom, Claim~\ref{disjoint claim}, and Modus Ponens inference rule,
$$
\E_{D_1}\chi_1\vdash\E_{ D_2}\chi_2\to\E_{ D_1\cup D_2}(\dots(\chi_n\to(\neg\wedge Z'\to\psi))\dots).
$$
By Modus Ponens inference rule,
$$
\E_{D_1}\chi_1,\E_{ D_2}\chi_2\vdash\E_{ D_1\cup D_2}(\dots(\chi_n\to(\neg\wedge Z'\to\psi))\dots).
$$
By repeating the previous two steps $(n-2)$ times,
$$
\E_{D_1}\chi_1,\E_{ D_2}\chi_2,\dots,\E_{ D_n}\chi_n\vdash\E_{ D_1\cup D_2\cup\dots \cup D_n}(\neg\wedge Z'\to\psi).
$$
Recall that $\E_{D_1}\chi_1,\E_{D_2}\chi_2,\dots, \E_{D_n}\chi_n\in X$ by the choice of $\E_{D_1}\chi_1$, \dots, $\E_{D_n}\chi_n$. Thus,
$X\vdash \E_{D_1\cup D_2\cup\dots\cup D_n}(\neg\wedge Z'\to\psi)$. 
Hence, because $D_1,\dots,D_n\subseteq C$, by Lemma~\ref{subset lemma E}, $X\vdash \E_{C}(\neg\wedge Z'\to\psi)$. Then, $X\vdash \E_{C}\psi$ by Epistemic Determinicity axiom and statement~(\ref{may eq}). Since the set $X$ is consistent, this contradicts the assumption $\neg\E_C\psi\in X$ of the lemma.
\end{proof}

\subsection{Complete Harmony}\label{complete harmony subsection}

\begin{definition}\label{complete harmony definition}
A pair in harmony $(X,Y)$ is in {\em complete} harmony if for each $\phi\in\Phi$ either $\neg\S_\varnothing\phi\in X$ or $\phi\in Y$.
\end{definition}
\begin{lemma}\label{complete harmony lemma}
For each pair in harmony $(X,Y)$, there is a pair in complete harmony $(X',Y')$ such that $X\subseteq X'$ and $Y\subseteq Y'$.
\end{lemma}
\begin{proof}
Recall that the set of agent $\mathcal{A}$ is finite and the set of propositional variables is countable. Thus, the set of all formulae $\Phi$ is also countable.
Let $\phi_1,\phi_2,\dots$ be an enumeration of all formulae in $\Phi$. We define two chains of sets $X_1\subseteq X_2\subseteq \dots$ and $Y_1\subseteq Y_2\subseteq \dots$ such that pair $(X_n,Y_n)$ is in harmony for each $n\ge 1$. These two chains are defined recursively as follows: 
\begin{enumerate}
    \item $X_1=X$ and $Y_1=Y$,
    \item if pair $(X_n,Y_n)$ is in harmony, then, by Lemma~\ref{harmony step}, either pair $(X_n\cup\{\neg\S_\varnothing\phi_n\},Y_n)$ or pair $(X_n,Y_n\cup\{\phi_n\})$ is in harmony. Let $(X_{n+1},Y_{n+1})$ be $(X_n\cup\{\neg\S_\varnothing\phi_n\},Y_n)$ in the former case and $(X_n,Y_n\cup\{\phi_n\})$ in the latter case.
\end{enumerate}
Let $X'=\bigcup_nX_n$ and $Y'=\bigcup_n Y_n$. Note that $X=X_1\subseteq X'$ and $Y=Y_1\subseteq Y'$.

We next show that pair $(X',Y')$ is in harmony. Suppose the opposite. Then, by Definition~\ref{harmony}, there is a finite set $Y''\subseteq Y'$ such that $X'\vdash\S_\varnothing\neg\wedge Y''$. Since a deduction uses only finitely many assumptions, there must exist $n_1\ge 1$ such that 
\begin{equation}\label{april eq}
X_{n_1} \vdash\S_\varnothing\neg\wedge Y''.
\end{equation}
At the same time, since set $Y''$ is finite, there must exist $n_2\ge 1$ such that $Y''\subseteq Y_{n_2}$. Let $n=\max\{n_1,n_2\}$. Note that $\neg\wedge Y''\to\neg\wedge Y_n$ is a tautology because $Y''\subseteq Y_{n_2}\subseteq Y_n$. Thus,  $\vdash\S_\varnothing(\neg\wedge Y''\to\neg\wedge Y_n)$ by Lemma~\ref{s necessitation}. Then, $\vdash\S_\varnothing\neg\wedge Y''\to\S_\varnothing\neg\wedge Y_n$ by  Cooperation axiom and Modus Ponens inference rule. Hence, $X_{n_1}\vdash\S_\varnothing\neg\wedge Y_n$ due to statement~(\ref{april eq}). 
Thus,
$X_{n}\vdash\S_\varnothing\neg\wedge Y_n$, because $X_{n_1}\subseteq X_n$. Then, pair $(X_n,Y_n)$ is not in harmony, which contradicts the choice of pair $(X_n,Y_n)$. Therefore, pair $(X',Y')$ is in harmony.

We finally show that pair $(X',Y')$ is in complete harmony. Indeed, consider any $\phi\in \Phi$. Since $\phi_1,\phi_2,\dots$ is an enumeration of all formulae in $\Phi$, there must exist $k\ge 1$ such that $\phi=\phi_k$. Then, by the choice of pair $(X_{k+1},Y_{k+1})$, either $\neg\S_\varnothing\phi=\neg\S_\varnothing\phi_k\in X_{k+1}\subseteq X'$ or $\phi=\phi_k\in Y_{k+1}\subseteq Y'$. Therefore, pair $(X',Y')$ is in complete harmony.
\end{proof}

\subsection{Canonical Epistemic Transition System}\label{ets section}

The construction of a canonical model, called the {\em canonical epistemic transition system}, for the proof of the completeness is based on the ``unravelling" technique~\cite{s75slfm}. Informally, epistemic states in this system are nodes in a tree. In this tree, each node is labeled with a maximal consistent set of formulae and each edge is labeled with a coalition. Formally, epistemic states are defined as sequences representing paths in such a tree. In the rest of this section we fix a maximal consistent set of formulae $X_0$ and define a canonical epistemic transition system $ETS(X_0)=(W,\{\sim_a\}_{a\in \mathcal{A}},V,M,\pi)$.

\begin{definition}\label{canonical worlds}
The set of epistemic states $W$ consists of all finite sequences $X_0,C_1,X_1,C_2,\dots,C_n,X_n$, such that
\begin{enumerate}
    \item $n\ge 0$,
    \item $X_i$ is a maximal consistent subset of $\Phi$ for each $i\ge 1$,
    \item $C_i$ is a coalition for each $i\ge 1$,
    \item $\{\phi\;|\;\K_{C_i}\phi\in X_{i-1}\}\subseteq X_i$ for each $i\ge 1$.
\end{enumerate}
\end{definition}

We say that two nodes of the tree are indistinguishable to an agent $a$ if every edge along the unique path connecting these two nodes is labeled with a coalition containing agent $a$.

\begin{definition}\label{canonical sim}
For any state $w=X_0,C_1,X_1,C_2,\dots,C_n,X_n$ and any state $w'=X_0,C'_1,X'_1,C'_2,\dots,C'_m,X'_m$, let $w\sim_a w'$ if there is an integer $k$ such that
\begin{enumerate}
    \item $0\le k\le\min\{n,m\}$,
    \item $X_i=X'_i$ for each $i$ such that $1\le i\le k$,
    \item $C_i=C'_i$ for each $i$ such that $1\le i\le k$,
    \item $a\in C_i$ for each $i$ such that $k<i\le n$,
    \item $a\in C'_i$ for each $i$ such that $k<i\le m$.
\end{enumerate}
\end{definition}

For any state $w=X_0,C_1,X_1,C_2,\dots,C_n,X_n$, by $hd(w)$ we denote the set $X_n$. The abbreviation $hd$ stands for ``head".

\begin{lemma}\label{down lemma}
For any $w=X_0,C_1,X_1,C_2,\dots,C_n,X_n\in W$ and any integer $k\le n$, if $\K_C\phi\in X_n$ and $C\subseteq C_i$ for each integer $i$ such that $k<i\le n$, then $\K_C\phi\in X_k$.
\end{lemma}
\begin{proof}
Suppose that there is $k\le n$ such that $\K_C\phi\notin X_k$. Let $m$ be the maximal such $k$. Note that $m<n$ due to the assumption $\K_C\phi\in X_n$ of the lemma. Thus, $m< m+1\le n$.

Assumption $\K_C\phi\notin X_{m}$ implies $\neg\K_C\phi\in X_m$ due to the maximality of the set $X_{m}$. Hence, $X_{m}\vdash \K_C\neg\K_C\phi$ by Negative Introspection axiom. Thus, $X_{m}\vdash \K_{C_{m+1}}\neg\K_C\phi$ by Monotonicity axiom and the assumption $C\subseteq C_{m+1}$ of the lemma (recall that $m+1\le n$). 
Then, $\K_{C_{m+1}}\neg\K_C\phi\in X_{m}$ due to the maximality of the set $X_{m}$.  Hence, $\neg\K_C\phi\in X_{m+1}$ by Definition~\ref{canonical worlds}.
Thus, $\K_C\phi\notin X_{m+1}$ due to the consistency of the set $X_{m+1}$, which is a contradiction with the choice of integer $m$.
\end{proof}

\begin{lemma}\label{up lemma}
For any $w=X_0,C_1,X_1,C_2,\dots,C_n,X_n\in W$ and any integer $k\le n$, if $\K_C\phi\in X_{k}$ and $C\subseteq C_i$ for each integer $i$ such that $k<i\le n$, then $\phi\in X_n$.
\end{lemma}
\begin{proof}
We prove the lemma by induction on the distance between $n$ and $k$. In the base case $n=k$. Then the assumption $\K_C\phi\in X_{n}$ implies $X_n\vdash\phi$ by Truth axiom. Therefore, $\phi\in X_n$ due to the maximality of set $X_n$.

Suppose that $k<n$. Assumption $\K_C\phi\in X_{k}$ implies $X_k\vdash \K_C\K_C\phi$ by Lemma~\ref{positive introspection lemma}. Thus, $X_k\vdash \K_{C_{k+1}}\K_C\phi$ by Monotonicity axiom, the condition $k<n$ of the inductive step, and the assumption $C\subseteq C_{k+1}$ of the lemma. Then, $\K_{C_{k+1}}\K_C\phi\in X_k$ by the maximality of set $X_k$. 
Hence, $\K_C\phi\in X_{k+1}$ by Definition~\ref{canonical worlds}. Therefore, $\phi\in X_n$ by the induction hypothesis. 
\end{proof}

\begin{lemma}\label{k child all}
If $\K_C\phi\in hd(w)$ and $w\sim_Cw'$, then $\phi\in hd(w')$.
\end{lemma}
\begin{proof}
The statement follows from Lemma~\ref{down lemma}, Lemma~\ref{up lemma}, and Definition~\ref{canonical sim} because there is a unique path between any two nodes in a tree.
\end{proof}

At the beginning of Section~\ref{consistent sets section}, we discussed that if a parent node contains a modal formula $\neg\Box\psi$, then it must have a child node containing formula $\neg\psi$. Lemma~\ref{k consist} in Section~\ref{consistent sets section} provides a foundation for constructing such a child node for modality $\K_C$. The proof of the next lemma describes the construction of the child node for this modality.

\begin{lemma}\label{k child exists}
If $\K_C\phi\notin hd(w)$, then there is an epistemic state $w'\in W$ such that $w\sim_C w'$ and $\phi\notin hd(w')$.
\end{lemma}
\begin{proof}
Assumption $\K_C\phi\notin hd(w)$ implies that $\neg\K_C\phi\in hd(w)$ due to the maximality of the set $hd(w)$.
Thus, by Lemma~\ref{k consist}, set 
$Y_0=\{\neg\phi\}\cup\{\psi\;|\;\K_C\psi\in hd(w)\}$
is consistent. Let $Y$ be a maximal consistent extension of set $Y_0$ and $w'$ be sequence $w,C,Y$. In other words, sequence $w'$ is an extension of sequence $w$ by two additional elements: $C$ and $Y$. Note that $w'\in W$ due to Definition~\ref{canonical worlds} and the choice of set $Y_0$. Furthermore, $w\sim_C w'$ by Definition~\ref{canonical sim}. To finish the proof, we need to show that $\phi\notin hd(w')$. Indeed, $\neg\phi\in Y_0\subseteq Y=hd(w')$ by the choice of $Y_0$. Therefore, $\phi\notin hd(w')$ due to the consistency of the set $hd(w')$.
\end{proof}

In the next two definitions we specify the domain of votes and the vote aggregation mechanism of the canonical transition system. Informally, a vote $(\phi,w)$ of each agent consists of two components: the actual vote $\phi$ and a key $w$. The actual vote $\phi$ is a formula from $\Phi$  in support of what the agent votes. Recall that the agent does not know in which exact state the system is, she only knows the equivalence class of this state with respect to the indistinguishability relation. The key $w$ is the agent's guess of the epistemic state where the system is. Informally, agent's vote has more power to force the formula to be satisfied in the next state if she guesses the current state correctly.

Although each agent is free to vote for any formula she likes, the vote aggregation mechanism would grant agent's wish only under certain circumstances. Namely, if the system is in state $w$ and set $hd(w)$ contains formula $\S_C\phi$, then the mechanism guarantees that formula $\phi$ is satisfied in the next state as long as each member of coalition $C$ votes for formula $\phi$ and correctly guesses the current epistemic state. In other words, in order for formula $\phi$ to be guaranteed in the next state all members of the coalition $C$ must cast vote $(\phi,w)$. This means that if $\S_C\phi\in hd(w)$, then coalition $C$ has a strategy to force $\phi$ in the next state. Since the strategy requires each member of the coalition to guess correctly the current state, such a strategy is not a know-how strategy.  

The vote aggregation mechanism is more forgiving if the epistemic state $w$ contains formula $\E_C\phi$. In this case the mechanism guarantees that formula $\phi$ is satisfied in the next state if all members of the coalition vote for formula $\phi$; it does not matter if they guess the current state correctly or not. This means that if $\E_C\phi\in hd(w)$, then coalition $C$ has a know-how strategy to force $\phi$ in the next state. The strategy consists in each member of the coalition voting for formula $\phi$ and specifying an arbitrary epistemic state as the key. 

Formal definitions of the domain of choices and of the vote aggregation mechanism in the canonical epistemic transition system are given below.

\begin{definition}\label{canonical domain}
The domain of choices $V$ is $\Phi\times W$.
\end{definition}

For any pair $u=(x,y)$, let $pr_1(u)=x$ and $pr_2(u)=y$.

\begin{definition}\label{canonical M}
The mechanism $M$ of the canonical model is the set of all tuples $(w,\{s_a\}_{a\in \mathcal{A}},w')$ such that for each formula $\phi\in\Phi$ and each coalition $C$,
\begin{enumerate}
    \item  if $\S_C\phi\in hd(w)$ and $s_a=(\phi,w)$ for each $a\in C$, then $\phi\in hd(w')$, and
    \item  if $\E_C\phi\in hd(w)$ and $pr_1(s_a)=\phi$ for each $a\in C$, then $\phi\in hd(w')$.
\end{enumerate}
\end{definition}

The next two lemmas prove that the vote aggregation mechanism specified in Definition~\ref{canonical M} acts as discussed in the informal description given earlier.

\begin{lemma}\label{s child all}
Let $w,w'\in W$ be epistemic states, $\S_C\phi\in hd(w)$ be a formula, and $\mathbf{s}=\{s_a\}_{a\in C}$ be a strategy profile of coalition $C$. If $w\to_{\mathbf s} w'$ and $s_a=(\phi,w)$ for each $a\in C$, then $\phi\in hd(w')$.
\end{lemma}
\begin{proof}
Suppose that $w\to_{\mathbf s} w'$. Thus, by Definition~\ref{single arrow}, there is a strategy profile $\mathbf{s'}=\{s'_a\}_{a\in\mathcal{A}}\in V^\mathcal{A}$ such that $s'_a=s_a$ for each $a\in C$ and $(w,\mathbf{s'},w')\in M$.
Therefore, $\phi\in hd(w')$ by Definition~\ref{canonical M} and the assumption $s_a=(\phi,w)$ for each $a\in C$. 
\end{proof}

\begin{lemma}\label{ks child all}
Let $w,w',w''\in W$ be epistemic states, $\E_C\phi\in hd(w)$ be a formula, and $\mathbf{s}=\{s_a\}_{a\in C}$ be a strategy profile of coalition $C$. If $w\sim_C w'$,  $w'\to_{\mathbf s} w''$, and $pr_1(s_a)=\phi$ for each $a\in C$, then $\phi\in hd(w'')$.
\end{lemma}
\begin{proof} Suppose that $\E_C\phi\in hd(w)$. Thus, $hd(w)\vdash \K_C\E_C\phi$ by Lemma~\ref{strategic positive introspection lemma}. Hence, $\K_C\E_C\phi\in hd(w)$ due to the maximality of the set $hd(w)$. Thus, $\E_C\phi\in hd(w')$ by Lemma~\ref{k child all} and the assumption $w\sim_C w'$.
By Definition~\ref{single arrow},
assumption $w'\to_{\mathbf s} w''$  implies that there is a strategy profile $\mathbf{s'}=\{s'_a\}_{a\in\mathcal{A}}$ such that $s'_a=s_a$ for each $a\in C$ and $(w',\mathbf{s'},w'')\in M$. 
Since $\E_C\phi\in hd(w')$, $pr_1(s'_a)=pr_1(s_a)=\phi$ for each $a\in C$, and $(w',\mathbf{s'},w'')\in M$, we have $\phi\in hd(w'')$ by Definition~\ref{canonical M}. 
\end{proof}

The lemma below provides a construction of a child node for modality $\S_C$. Although the proof follows the outline of the proof of Lemma~\ref{k child exists} for modality $\K_C$, it is significantly more involved because of the need to show that a transition from a parent node to a child node satisfies the constraints of the vote aggregation mechanism from Definition~\ref{canonical M}.  

\begin{lemma}\label{s child exists}
For any epistemic state $w\in W$, any formula $\neg\S_C\psi\in hd(w)$, and any strategy profile $\mathbf{s}=\{s_a\}_{a\in C}\in V^C$, there is a state $w'\in W$ such that $w\to_\mathbf{s} w'$ and $\psi\notin hd(w')$.
\end{lemma}
\begin{proof}
Let $Y_0$ be the following set of formulae 
\begin{eqnarray*}
\{\neg\psi\}\cup\{\phi\;|\; \exists D\subseteq C(\S_D\phi\in hd(w) \wedge\forall a\in D (pr_1(s_a)=\phi))\}.
\end{eqnarray*}
We first show that set $Y_0$ is consistent. Suppose the opposite. Thus, there must exist formulae  $\phi_1,\dots,\phi_n\in Y_0$ and subsets $D_1,\dots,D_n\subseteq C$ such that (i) $\S_{D_i}\phi_i\in hd(w)$ for each integer $i\le n$, (ii) $pr_1(s_a)=\phi_i$ for each $i\le n$ and each $a\in D_i$, and (iii) set $\{\neg\psi,\phi_1,\dots,\phi_n\}$ is inconsistent. Without loss of generality we can assume that formulae  $\phi_1,\dots,\phi_n$ are {\em pairwise distinct}. 

\begin{claim}\label{claim 1}
Sets $D_i$ and $D_j$ are disjoint for each $i\neq j$.
\end{claim}
\noindent{\sc Proof of Claim.}
Assume that $d\in D_i\cap D_j$, then $pr_1(s_d)=\phi_i$ and $pr_1(s_d)=\phi_j$. Hence, $\phi_i=\phi_j$, which contradicts the assumption that formulae $\phi_1,\dots,\phi_n$ are pairwise distinct. Therefore, sets $D_i$ and $D_j$ are disjoint for each $i\neq j$. \qed

By Lemma~\ref{i neq j S}, it follows from Claim~\ref{claim 1} that set $Y_0$ is consistent. Let $Y$ be any maximal consistent extension of $Y_0$ and $w'$ be the sequence $w,\varnothing,Y$. In other words, $w'$ is an extension of sequence $w$ by two additional elements: $\varnothing$ and $Y$.

\begin{claim}
$w'\in W$.
\end{claim}
\noindent{\sc Proof of Claim.}
By Definition~\ref{canonical worlds}, it suffices to show that, for each formula $\phi\in \Phi$, if $\K_\varnothing\phi\in hd(w)$, then $\phi\in Y$. Indeed, suppose that $\K_\varnothing\phi\in hd(w)$. Thus, $hd(w)\vdash\E_\varnothing\phi$ by Empty Coalition axiom. Hence, $hd(w)\vdash\S_\varnothing\phi$ by Strategic Truth axiom. Then, $\S_\varnothing\phi\in hd(w)$ due to the maximality of set $hd(w)$. Therefore, $\phi\in Y_0\subseteq Y$ by the choice of sets $Y_0$ and $Y$.
\qed

Let $\top$ be any propositional tautology. For example, $\top$ could be formula $\psi\to\psi$. Define strategy profile $\mathbf{s'}=\{s'_a\}_{a\in \mathcal{A}}$ as follows
\begin{equation}\label{canonical vote}
s'_a=
\begin{cases}
s_a, & \mbox{ if } a\in C,\\
(\top,w), & \mbox{ otherwise}.
\end{cases}
\end{equation}

\begin{claim}\label{claim 2}  
For any formula $\phi\in\Phi$ and any $D\subseteq \mathcal{A}$, if  
$\S_D\phi\in hd(w)$ and $s'_a=(\phi,w)$ for each $a\in D$, then $\phi\in hd(w')$.
\end{claim}
\noindent{\sc Proof of Claim.}
 Consider any formula $\phi\in\Phi$ and any set $D\subseteq \mathcal{A}$ such that $\S_D\phi\in hd(w)$ and $s'_a=(\phi,w)$ for each agent $a\in D$. We need to show that $\phi\in hd(w')$.
 
 \noindent{\em Case 1:} $D\subseteq C$. In this case, $s_a=s'_a=(\phi,w)$ for each $a\in D$ by definition~(\ref{canonical vote}). Thus, $\phi\in Y_0\subseteq Y=hd(w')$ by the choice of set $Y_0$.

\noindent{\em Case 2:} There is $a_0\in D$ such that $a_0\notin C$. Then, $s'_{a_0}=(\top,w)$ by definition~(\ref{canonical vote}). Note that $s'_{a_0}=(\phi,w)$ by the choice of the set $D$. Thus, $(\top,w)=(\phi,w)$. Hence, formula $\phi$ is the tautology $\top$. Therefore, $\phi\in hd(w')$ because set $hd(w')$ is maximal. \qed

\begin{claim}\label{claim 3}  
For any formula $\phi\in\Phi$ and any $D\subseteq \mathcal{A}$, if       $\E_D\phi\in hd(w)$ and $pr_1(s'_a)=\phi$ for each $a\in D$, then $\phi\in hd(w')$.
\end{claim}

\noindent{\sc Proof of Claim.}
 Consider any formula $\phi\in\Phi$ and any set $D\subseteq \mathcal{A}$ such that $\E_D\phi\in hd(w)$ and $pr_1(s'_a)=\phi$ for each agent $a\in D$. We need to show that $\phi\in hd(w')$.
 
 \noindent{\em Case 1:} $D\subseteq C$. In this case, $pr_1(s_a)=pr_1(s'_a)=\phi$ for each agent $a\in D$ by definition~(\ref{canonical vote}) and the choice of set $D$. Thus, $\phi\in Y_0\subseteq Y=hd(w')$ by the choice of set $Y_0$.

\noindent{\em Case 2:} There is agent $a_0\in D$ such that $a_0\notin C$. Then, $s'_{a_0}=(\top,w)$ by definition~(\ref{canonical vote}). Note that $pr_1(s'_{a_0})=\phi$ by the choice of set $D$. Thus, $\top=\phi$. Hence, formula $\phi$ is the tautology $\top$. Therefore, $\phi\in hd(w')$ because set $hd(w')$ is maximal. \qed

 By Definition~\ref{canonical M}, Claim~\ref{claim 2} and Claim~\ref{claim 3} together imply that $(w,\mathbf{s'},w')\in M$. Hence, $w\to_{\mathbf s} w'$ by Definition~\ref{single arrow} and definition~(\ref{canonical vote}). To finish the proof of the lemma, note that $\psi\notin hd(w')$ because set $hd(w')$ is consistent and $\neg\psi\in Y_0\subseteq Y=hd(w')$. 
\end{proof}

 The next lemma shows the construction of a child node for modality $\E_C$. The proof is similar to the proof of Lemma~\ref{s child exists} except that, instead of constructing a single child node, we construct two sibling nodes that are in complete harmony. The intuition was discussed at the beginning of Section~\ref{harmony subsection}. 

\begin{lemma}\label{e child exists}
For any state $w\in W$, any formula $\neg\E_C\psi\in hd(w)$, and any strategy profile $\mathbf{s}=\{s_a\}_{a\in C}\in V^C$, there are epistemic states $w',w''\in W$ such that $\psi\notin hd(w'')$, $w\sim_C w'$, and $w'\to_\mathbf{s} w''$.
\end{lemma}
\begin{proof}
By Definition~\ref{canonical domain}, for each $a\in C$, vote $s_a$ is a pair. Let 
\begin{eqnarray*}
&&\hspace{-7mm}Y = \{\phi\;|\; \K_C\phi\in hd(w)\}, \mbox{ and}\\
&&\hspace{-7mm}Z = \{\neg\psi\} \cup \{\phi\;|\;\exists D\subseteq C\;(\E_D\phi\in hd(w) \wedge \forall a\in D\;(pr_1(s_a)=\phi))\}.
\end{eqnarray*}
By Lemma~\ref{i neq j E} where $f(x)=pr_1(s_x)$, pair $(Y,Z)$ is in harmony. By Lemma~\ref{complete harmony lemma}, there is a pair $(Y',Z')$ in complete harmony such that $Y\subseteq Y'$ and $Z\subseteq Z'$. By Lemma~\ref{harmony con 1} and Lemma~\ref{harmony con 2}, sets $Y'$ and $Z'$ are consistent. Let $Y''$ and $Z''$ be maximal consistent extensions of sets $Y'$ and $Z'$, respectively.

Recall that set $\mathcal{A}$ is finite. Thus, set $C\subseteq\mathcal{A}$ is also finite.
Let integer $n$ be the cardinality of set $C$. Consider $(n+1)$ sequences $w_1,w_2,\dots,w_{n+1}$, where sequence $w_k$ is an extension of sequence $w$ that adds $2k$ additional elements:
\begin{eqnarray*}
&&w_1=w,C,Y''\\
&&w_2=w,C,Y'',C,Y''\\
&&w_3=w,C,Y'',C,Y'',C,Y''\\
&&\dots\\
&&w_{n+1}=w,\underbrace{C,Y'',\dots,C,Y''}_{2(n+1)\mbox{ elements}}.
\end{eqnarray*} 
\begin{claim}\label{alpha}
$w_k\in W$ for each $k\le n+1$.
\end{claim}
\noindent{\sc Proof of Claim.} We prove the claim by induction on integer $k$. 

\noindent{\em Base Case:} By Definition~\ref{canonical worlds}, it suffices to show that if $\K_C\phi\in hd(w)$, then $\phi\in hd(w_1)$. Indeed, if $\K_C\phi\in hd(w)$, then $\phi\in Y$ by the choice of set $Y$. Therefore, $\phi\in Y\subseteq Y'\subseteq Y''=hd(w_1)$.

\noindent{\em Induction Step:} By Definition~\ref{canonical worlds}, it suffices to show that if $\K_C\phi\in hd(w_k)$, then $\phi\in hd(w_{k+1})$ for each $k\ge 1$. In other words, we need to prove that if $\K_C\phi\in Y''$, then $\phi\in Y''$, which follows from Truth axiom and the maximality of set $Y''$.
\qed

By the pigeonhole principle, there is $i_0\le n$ such that $pr_2(s_a)\neq w_{i_0}$ for all $a\in C$. Let $w'$ be epistemic state $w_{i_0}$. Thus,
\begin{equation}\label{choice of w'}
pr_2(s_a)\neq w' \mbox{ for each $a\in C$}.
\end{equation}

Let $w''$ be the sequence $w,\varnothing, Z''$. In other words, sequence $w''$ is an extension of sequence $w$ by two additional elements: $\varnothing$ and $Z''$.  Finally, let strategy profile $\mathbf{s'}=\{s'_a\}_{a\in \mathcal{A}}$ be defined as follows
\begin{equation}\label{canonical vote 2}
s'_a=
\begin{cases}
s_a, & \mbox{ if } a\in C,\\
(\top,w'), & \mbox{ otherwise}.
\end{cases}
\end{equation}

\begin{claim}\label{beta}
$w''\in W$.
\end{claim}
\noindent{\sc Proof of Claim.}
By Definition~\ref{canonical worlds}, it suffices to show that if $\K_\varnothing\phi\in hd(w)$, then $\phi\in hd(w'')$ for each formula $\phi\in\Phi$. Indeed, by Empty Coalition axiom, assumption $\K_\varnothing\phi\in hd(w)$ implies that $hd(w)\vdash \E_\varnothing\phi$. Hence, $\E_\varnothing\phi\in hd(w)$ by the maximality of the set $hd(w)$. Thus, $\phi\in Z$ by the choice of set $Z$. Therefore, $\phi\in Z\subseteq Z'\subseteq Z''=hd(w'')$.
\qed

\begin{claim}\label{gamma}
$w\sim_C w'$.
\end{claim}
\noindent{\sc Proof of Claim.}
By Definition~\ref{canonical sim}, $w\sim_C w_i$ for each integer $i\le n+1$. In particular, $w\sim_C w_{i_0}=w'$.
\qed

\begin{claim}\label{delta}
$\psi\notin hd(w'')$.
\end{claim}
\noindent{\sc Proof of Claim.}
Note that $\neg\psi\in Z$ by the choice of set $Z$. Thus, $\neg\psi\in Z\subseteq Z'\subseteq Z'' = hd(w'')$. Therefore, $\psi\notin hd(w'')$ due to the consistency of the set $hd(w'')$.
\qed

\begin{claim}\label{winter}
Let $\phi$ be a formula in $\Phi$ and $D$ be a subset of $\mathcal{A}$. If $\S_D\phi\in hd(w')$ and $s'_a=(\phi,w')$ for each $a\in D$, then $\phi\in hd(w'')$.
\end{claim}

\noindent{\sc Proof of Claim.} Note that either set $D$ is empty or it contains an element $a_0$. In the latter case, element $a_0$ either belongs or does not belong to set $C$.

\noindent{\em Case I:} $D=\varnothing$. Recall that pair $(Y',Z')$ is in complete harmony. Thus, by Definition~\ref{complete harmony definition}, either $\neg\S_\varnothing\phi\in Y'\subseteq Y''=hd(w')$ or $\phi\in Z'\subseteq Z''=hd(w'')$. Assumption $\S_D\phi\in hd(w')$ implies that $\neg\S_\varnothing\phi\notin hd(w')$ due to the consistency of the set $hd(w')$ and the assumption $D=\varnothing$ of the case. Therefore, $\phi\in hd(w'')$.

\noindent{\em Case II:} there is an element $a_0\in C\cap D$. Thus, $a_0\in C$. Hence, $pr_2(s_{a_0})\neq w'$ by inequality (\ref{choice of w'}). Then, $s_{a_0}\neq (\phi,w')$. Thus, $s'_{a_0}\neq (\phi,w')$ by definition~(\ref{canonical vote 2}). Recall that $a_0\in C\cap D\subseteq D$.
This contradicts the assumption that $s'_a=(\phi,w')$ for each $a\in D$.

\noindent{\em Case III:} there is an element $a_0\in D\setminus C$. Thus, $s'_{a_0}=(\top,w')$ by definition~(\ref{canonical vote 2}). At the same time, $s'_{a_0}=(\phi,w')$ by the second assumption of the claim. Hence, formula $\phi$ is the propositional tautology $\top$. Therefore, $\phi\in hd(w'')$ due to the maximality of the set $hd(w'')$.
\qed

\begin{claim}\label{summer}
Let $\phi$ be a formula in $\Phi$ and $D$ be a subset of $\mathcal{A}$.
If $\E_D\phi\in hd(w')$ and $pr_1(s'_a)=\phi$ for each $a\in D$, then $\phi\in hd(w'')$.
\end{claim}
\noindent{\sc Proof of Claim.}

\noindent{\em Case I:} $D\subseteq C$. Suppose that $pr_1(s'_a)=\phi$ for each $a\in D$ and $\E_D\phi\in hd(w')$. Thus, $\phi\in Z$ by the choice of set $Z$. Therefore, $\phi\in Z\subseteq Z'\subseteq Z''=hd(w'')$. 

\noindent{\em Case II:} $D\nsubseteq C$. Consider any $a_0\in D\setminus C$. Note that $s'_{a_0}=(\top,w')$ by definition~(\ref{canonical vote 2}). At the same time, $pr_1(s'_{a_a})=\phi$ by the second assumption of the claim. Hence, formula $\phi$ is the propositional tautology $\top$. Therefore, $\phi\in hd(w'')$ due to the maximality of the set $hd(w'')$.
\qed

Claim~\ref{winter} and Claim~\ref{summer}, by Definition~\ref{canonical M}, imply that $(w',\{s'_a\}_{a\in \mathcal{A}},w'')\in M$. Thus, $w'\to_\mathbf{s} w''$ by Definition~\ref{single arrow} and definition~(\ref{canonical vote 2}). 
This together with Claim~\ref{alpha}, Claim~\ref{beta}, Claim~\ref{gamma}, and Claim~\ref{delta} completes the proof of the lemma.
\end{proof}

\begin{definition}\label{canonical pi}
$\pi(p)=\{w\in W\;|\; p\in hd(w)\}$.
\end{definition}

This concludes the definition of tuple  $(W,\{\sim_a\}_{a\in \mathcal{A}},V,M,\pi)$.

\begin{lemma}
Tuple $(W,\{\sim_a\}_{a\in \mathcal{A}},V,M,\pi)$ is an epistemic transition system.
\end{lemma}
\begin{proof}
By Definition~\ref{transition system}, it suffices to show that for each $w\in W$ and each $\mathbf{s}\in V^\mathcal{A}$ there is $w'\in W$ such that $(w,\mathbf{s},w')\in M$. 

Recall that set $\mathcal{A}$ is finite. Thus, $\vdash\neg\S_\mathcal{A}\bot$ by Nontermination axiom. Hence, $\neg\S_\mathcal{A}\bot\in hd(w)$. By Lemma~\ref{s child exists}, there is $w'\in W$ such that $w\to_{\mathbf{s}}w'$. Therefore, $(w,\mathbf{s},w')\in M$ by Definition~\ref{single arrow}.  
\end{proof}

\begin{lemma}\label{main induction lemma}
$w\Vdash\phi$ iff $\phi\in hd(w)$ for each epistemic state $w\in W$ and each formula $\phi\in\Phi$.
\end{lemma}
\begin{proof}
We prove the lemma by induction on the structural complexity of formula $\phi$. If formula $\phi$ is a propositional variable, then the required follows from Definition~\ref{sat} and Definition~\ref{canonical pi}. The cases of formula $\phi$ being a negation or an implication follow from Definition~\ref{sat}, and the maximality and the consistency of the set $hd(w)$ in the standard way.

Let formula $\phi$ have the form $\K_C\psi$. 

\noindent$(\Rightarrow)$ Suppose that $\K_C\psi\notin hd(w)$. Then, by Lemma~\ref{k child exists}, there is $w'\in W$ such that $w\sim_C w'$ and $\psi\notin hd(w')$. Hence, $w'\nVdash \psi$ by the induction hypothesis. Therefore, $w\nVdash\K_C\psi$ by Definition~\ref{sat}.

\noindent$(\Leftarrow)$ Assume that $\K_C\psi\in hd(w)$. Consider any $w'\in W$ such that $w\sim_C w'$. By Definition~\ref{sat}, it suffices to show that $w'\Vdash\psi$. Indeed, $\psi\in hd(w')$ by Lemma~\ref{k child all}. Therefore, by the induction hypothesis, $w'\Vdash\psi$.

Let formula $\phi$ have the form $\S_C\psi$.

\noindent$(\Rightarrow)$ Suppose that $\S_C\psi\notin hd(w)$. Then, $\neg\S_C\psi\in hd(w)$ due to the maximality of the set $hd(w)$. Hence, by Lemma~\ref{s child exists}, for any strategy profile $\mathbf{s}\in V^C$, there is an epistemic state $w'\in W$ such that $w\to_\mathbf{s}w'$ and $\psi\notin hd(w')$. Thus, by the induction hypothesis, for any strategy profile $\mathbf{s}\in V^C$, there is a state $w'\in W$ such that $w\to_\mathbf{s}w'$ and $w'\nVdash\psi$. Then, $w\nVdash\S_C\psi$ by Definition~\ref{sat}.

\noindent$(\Leftarrow)$ Assume that $\S_C\psi\in hd(w)$. Consider strategy profile ${\mathbf s}=\{s_a\}_{a\in C}\in V^C$ such that $s_a=(\psi,w)$ for each $a\in C$. By Lemma~\ref{s child all}, for any epistemic state $w'\in W$, if  $w\to_{\mathbf s} w'$, then $\psi\in hd(w')$. Hence, by the induction hypothesis, for any epistemic state $w'\in W$, if  $w\to_{\mathbf s} w'$, then $w'\Vdash \psi$. Therefore, $w\Vdash\S_C\psi$ by Definition~\ref{sat}.

Finally, let formula $\phi$ have the form $\E_C\psi$. 

\noindent$(\Rightarrow)$ Suppose that $\E_C\psi\notin hd(w)$. Then, $\neg\E_C\psi\in hd(w)$ due to the maximality of the set $hd(w)$. Hence, by Lemma~\ref{e child exists}, for any strategy profile $\mathbf{s}\in V^C$, there are epistemic states $w',w''\in W$ such that $w\sim_C w'$, $w'\to_\mathbf{s}w''$, and $\psi\notin hd(w'')$. Thus, $w''\nVdash\psi$ by the induction hypothesis. Therefore, $w\nVdash \E_C\psi$ by Definition~\ref{sat}.

\noindent$(\Leftarrow)$ Assume that $\E_C\psi\in hd(w)$. Consider a strategy profile ${\mathbf s}=\{s_a\}_{a\in C}\in V^C$ such that $s_a=(\psi,w)$ for each $a\in C$. By Lemma~\ref{ks child all}, for all epistemic states $w',w''\in W$, if  $w\sim_C w'$, and $w'\to_{\mathbf s} w''$, then $\psi\in hd(w'')$. Hence, by the induction hypothesis, $w''\Vdash \psi$. Therefore, $w\Vdash\E_C\psi$ by Definition~\ref{sat}.
\end{proof}

\subsection{Completeness: the Final Step}

To finish the proof of Theorem~\ref{completeness theorem} stated at the beginning of Section~\ref{completeness section}, suppose that $\nvdash\phi$. Let $X_0$ be any maximal consistent subset of set $\Phi$ such that $\neg\phi\in X_0$. Consider the canonical epistemic transition system $ETS(X_0)$ defined in Section~\ref{ets section}.
Let $w$ be the single-element sequence $X_0$. Note that $w\in W$ by Definition~\ref{canonical worlds}. Thus, $w\Vdash \neg\phi$ by Lemma~\ref{main induction lemma}. Therefore, $w\nVdash\phi$ by Definition~\ref{sat}. 

\section{Conclusion}\label{conclusion section}

In this article we proposed a sound and complete logic system that captures an interplay between the  distributed knowledge, coalition strategies, and how-to strategies. In the future work we hope to explore know-how strategies of non-homogeneous coalitions in which different members contribute differently to the goals of the coalition. For example, ``incognito" members of a coalition might contribute only by sharing information, while ``open" members also contribute by voting.

\bibliography{sp} 

\begin{thebibliography}{10}

\bibitem{aa12aamas}
Thomas {\AA}gotnes and Natasha Alechina.
\newblock Epistemic coalition logic: completeness and complexity.
\newblock In {\em Proceedings of the 11th International Conference on
  Autonomous Agents and Multiagent Systems-Volume 2}, pages 1099--1106.
  International Foundation for Autonomous Agents and Multiagent Systems, 2012.

\bibitem{abvs10jal}
Thomas {\AA}gotnes, Philippe Balbiani, Hans van Ditmarsch, and Pablo Seban.
\newblock Group announcement logic.
\newblock {\em Journal of Applied Logic}, 8(1):62 -- 81, 2010.

\bibitem{avw09ai}
Thomas {\AA}gotnes, Wiebe van~der Hoek, and Michael Wooldridge.
\newblock Reasoning about coalitional games.
\newblock {\em Artificial Intelligence}, 173(1):45 -- 79, 2009.

\bibitem{aeh75sigop}
Eralp~A Akkoyunlu, Kattamuri Ekanadham, and RV~Huber.
\newblock Some constraints and tradeoffs in the design of network
  communications.
\newblock In {\em ACM SIGOPS Operating Systems Review}, volume~9, pages 67--74.
  ACM, 1975.

\bibitem{ahk02}
Rajeev Alur, Thomas~A. Henzinger, and Orna Kupferman.
\newblock Alternating-time temporal logic.
\newblock {\em Journal of the ACM}, 49(5):672--713, 2002.

\bibitem{b14sr}
Francesco Belardinelli.
\newblock Reasoning about knowledge and strategies: Epistemic strategy logic.
\newblock In {\em Proceedings 2nd International Workshop on Strategic
  Reasoning, {SR} 2014, Grenoble, France, April 5-6, 2014}, volume 146 of {\em
  {EPTCS}}, pages 27--33, 2014.

\bibitem{b07ijcai}
Stefano Borgo.
\newblock Coalitions in action logic.
\newblock In {\em 20th International Joint Conference on Artificial
  Intelligence}, pages 1822--1827, 2007.

\bibitem{fhmv95}
Ronald Fagin, Joseph~Y. Halpern, Yoram Moses, and Moshe~Y. Vardi.
\newblock {\em Reasoning about knowledge}.
\newblock MIT Press, Cambridge, MA, 1995.

\bibitem{fhlw17ijcai}
Raul Fervari, Andreas Herzig, Yanjun Li, and Yanjing Wang.
\newblock Strategically knowing how.
\newblock In {\em 26th International Joint Conference on Artificial
  Intelligence (IJCAI-17), August 19-25, 2017}, 2017.
\newblock (to appear).

\bibitem{g01tark}
Valentin Goranko.
\newblock Coalition games and alternating temporal logics.
\newblock In {\em Proceedings of the 8th conference on Theoretical aspects of
  rationality and knowledge}, pages 259--272. Morgan Kaufmann Publishers Inc.,
  2001.

\bibitem{g78os}
James~N Gray.
\newblock Notes on data base operating systems.
\newblock In {\em Operating Systems}, pages 393--481. Springer, 1978.

\bibitem{ja07jancl}
Wojciech Jamroga and Thomas {\AA}gotnes.
\newblock Constructive knowledge: what agents can achieve under imperfect
  information.
\newblock {\em Journal of Applied Non-Classical Logics}, 17(4):423--475, 2007.

\bibitem{jv04fm}
Wojciech Jamroga and Wiebe van~der Hoek.
\newblock Agents that know how to play.
\newblock {\em Fundamenta Informaticae}, 63(2-3):185--219, 2004.

\bibitem{mn12tocl}
Sara~Miner More and Pavel Naumov.
\newblock Calculus of cooperation and game-based reasoning about protocol
  privacy.
\newblock {\em ACM Trans. Comput. Logic}, 13(3):22:1--22:21, August 2012.

\bibitem{nt17aamas}
Pavel Naumov and Jia Tao.
\newblock Coalition power in epistemic transition systems.
\newblock In {\em Proceedings of the 2017 International Conference on
  Autonomous Agents and Multiagent Systems}, pages 723--731, 2017.

\bibitem{nt17tark}
Pavel Naumov and Jia Tao.
\newblock Together we know how to achieve: An epistemic logic of know-how.
\newblock In {\em 16th conference on Theoretical Aspects of Rationality and
  Knowledge (TARK `17), July 24-26, 2017}, 2017.
\newblock (to appear).

\bibitem{p01illc}
Marc Pauly.
\newblock {\em Logic for Social Software}.
\newblock PhD thesis, Institute for Logic, Language, and Computation, 2001.

\bibitem{p02}
Marc Pauly.
\newblock A modal logic for coalitional power in games.
\newblock {\em Journal of Logic and Computation}, 12(1):149--166, 2002.

\bibitem{s75slfm}
Henrik Sahlqvist.
\newblock Completeness and correspondence in the first and second order
  semantics for modal logic.
\newblock {\em Studies in Logic and the Foundations of Mathematics},
  82:110--143, 1975.
\newblock (Proc. of the 3rd Scandinavial Logic Symposium, Uppsala, 1973).

\bibitem{sgvw06aamas}
Luigi Sauro, Jelle Gerbrandy, Wiebe van~der Hoek, and Michael Wooldridge.
\newblock Reasoning about action and cooperation.
\newblock In {\em Proceedings of the Fifth International Joint Conference on
  Autonomous Agents and Multiagent Systems}, AAMAS '06, pages 185--192, New
  York, NY, USA, 2006. ACM.

\bibitem{v01ber}
Johan Van~Benthem.
\newblock Games in dynamic-epistemic logic.
\newblock {\em Bulletin of Economic Research}, 53(4):219--248, 2001.

\bibitem{vw03sl}
Wiebe van~der Hoek and Michael Wooldridge.
\newblock Cooperation, knowledge, and time: Alternating-time temporal epistemic
  logic and its applications.
\newblock {\em Studia Logica}, 75(1):125--157, 2003.

\bibitem{vw05ai}
Wiebe van~der Hoek and Michael Wooldridge.
\newblock On the logic of cooperation and propositional control.
\newblock {\em Artificial Intelligence}, 164(1):81 -- 119, 2005.

\bibitem{w17synthese}
Yanjing Wang.
\newblock A logic of goal-directed knowing how.
\newblock {\em Synthese}.
\newblock (to appear).

\bibitem{w15lori}
Yanjing Wang.
\newblock A logic of knowing how.
\newblock In {\em Logic, Rationality, and Interaction}, pages 392--405.
  Springer, 2015.

\end{thebibliography}

\end{document}